\documentclass{article}

\PassOptionsToPackage{numbers, compress}{natbib}

\usepackage[final]{neurips_2023}

\usepackage[utf8]{inputenc} 
\usepackage[T1]{fontenc}    
\usepackage[pagebackref,breaklinks,colorlinks]{hyperref}
\usepackage{url}            
\usepackage{booktabs}       
\usepackage{amsfonts}       
\usepackage{nicefrac}       
\usepackage{microtype}      
\usepackage{comment}
\usepackage{amsmath}
\usepackage[ruled,vlined]{algorithm2e}
\usepackage{algpseudocode}
\usepackage{amsthm}
\usepackage{amssymb}
\usepackage{enumitem}
\usepackage[table]{xcolor}
\usepackage[export]{adjustbox}
\usepackage{wrapfig,lipsum,booktabs}
\usepackage{makecell}
\usepackage{caption}
\usepackage{float}

\usepackage{multirow}
\usepackage{arydshln}
\usepackage{bm}
\newcommand{\cmark}{\ding{51}}
\newcommand{\xmark}{\ding{55}}
\usepackage{pifont}

\definecolor{mygray}{gray}{.9}
\definecolor{mygray1}{gray}{.97}
\definecolor{mygreen}{RGB}{93,173,85}
\newcommand{\pub}[1]{{\color{gray}{\tiny{[{#1}]}}}}

\makeatletter
\newcommand{\thickhline}{
    \noalign {\ifnum 0=`}\fi \hrule height 1pt
    \futurelet \reserved@a \@xhline
}

\makeatletter
\newenvironment{fullitemize}
{
\vspace{-6pt}
\begin{itemize}[fullwidth]
    \setlength{\topsep}{0pt}
    \setlength{\itemsep}{0pt}
    \setlength{\parsep}{0pt}
    \setlength{\parskip}{0pt}
    \setlength{\partopsep}{0pt}
}
{\end{itemize}
\vspace{-6pt}}
\makeatother

\usepackage{listings}
\usepackage{setspace}

\definecolor{codegreen}{RGB}{79,126,127}
\definecolor{codedefine}{RGB}{153,54,159}
\definecolor{codevar}{RGB}{73,122,234}
\definecolor{codefunc}{RGB}{184,134,11}
\definecolor{codetype}{RGB}{58,95,205}
\definecolor{codecall}{RGB}{73,122,234}
\definecolor{codepro}{RGB}{212,96,80}
\definecolor{codedim}{RGB}{89,152,195}

\title{Label-efficient Segmentation via Affinity Propagation}

\author{Wentong Li$^{1*}$, Yuqian Yuan$^{1}$\thanks{Equal contribution}~~, Song Wang$^{1}$, Wenyu Liu$^{1}$,  \\ \textbf{Dongqi Tang$^{2}$, Jian Liu$^{2}$, Jianke Zhu$^{1}$\thanks{Correspondence author}, Lei Zhang$^{3}$} \\
  $^{1}$Zhejiang University \quad
  $^{2}$Ant Group  \quad
  $^{3}$The Hong Kong Polytechnical University\\
  \url{https://LiWentomng.github.io/apro/}}

\begin{document}

\maketitle

\begin{abstract} 
  Weakly-supervised segmentation with label-efficient sparse annotations has attracted increasing research attention to reduce the cost of laborious pixel-wise labeling process, while the pairwise affinity modeling techniques play an essential role in this task. Most of the existing approaches focus on using the local appearance kernel to model the neighboring pairwise potentials. However, such a local operation fails to capture the long-range dependencies and ignores the  topology of objects. 
  In this work, we formulate the affinity modeling as an affinity propagation process, and propose a local and a global pairwise affinity terms to generate accurate soft pseudo labels. An efficient algorithm is also developed to reduce significantly the computational cost. The proposed approach can be conveniently plugged into existing segmentation networks. Experiments on three typical label-efficient segmentation tasks, \textit{i.e.} box-supervised instance segmentation, point/scribble-supervised semantic segmentation and CLIP-guided semantic segmentation, demonstrate the superior performance of the proposed approach. 
\end{abstract}

\section{Introduction}
Segmentation is a widely studied problem in computer vision, aiming at generating a mask prediction for a given image, \textit{e.g.,} grouping each pixel to an object instance (\textit{instance segmentation}) or assigning each pixel a category label (\textit{semantic segmentation}). While having achieved promising performance, most of the existing  approaches are trained in a fully supervised manner, which heavily depend on the pixel-wise mask annotations, incurring tedious labeling costs~\cite{shen2023survey}.
Weakly-supervised methods have been proposed to reduce the dependency on dense pixel-wise labels with label-efficient sparse annotations, such as  points~\cite{bearman2016s,fan2022pointly,li2023point2mask},  scribbles~\cite{lin2016scribblesup, tang2018regularized, liang2022tree}, bounding boxes~\cite{cvpr2021boxinst, li2022box, CVPR2023MAL, li2022box2mask} and image-level labels~\cite{cvpr2018learning,ahn2019weakly,ru2022learning,li2022expansion}.
Such methods make dense segmentation more accessible with less annotation costs for new categories or scene types. 

Most of the existing weakly-supervised segmentation methods~\cite{nips2019-bbtp, ahn2019weakly,cvpr2021boxinst, iccv2021discobox, fan2022pointly,CVPR2023MAL} adopt the local appearance kernel to model the neighboring pairwise affinities, where spatially nearby pixels with similar color (\textit{i.g.}, LAB color space~\cite{cvpr2021boxinst,fan2022pointly} or RGB color space~\cite{ahn2019weakly, nips2019-bbtp, iccv2021discobox, CVPR2023MAL}) are likely to be in the same class. Though having proved to be effective, these methods suffer from two main limitations. First, the local operation cannot capture global context cues and capture long-range affinity dependencies. Second, the appearance kernel fails to take the intrinsic topology of objects into account, and lacks capability of detail preservation.

To address the first issue, one can directly enlarge the kernel size to obtain a large receptive filed. However, this will make the segmentation model insensitive to local details and increase the computational cost greatly. Some methods~\cite{ru2022learning,cvpr2018learning} model the long-range affinity via random walk~\cite{cvpr2017randomwalk}, but they cannot model the fine-grained semantic affinities.
As for the second issue, the tree-based approaches~\cite{liang2022tree,nips2019treev1} are able to preserve the geometric structures of objects, and employ the minimum spanning tree~\cite{kruskal1956shortest} to capture the pairwise relationship. However, the affinity  interactions with distant nodes will decay rapidly as the distance increases along the spanning tree, which still focuses on the local nearby regions.
LTF-V2~\cite{nips2020treev2} enables the long-range tree-based interactions but it fails to model the valid pairwise affinities for label-efficient segmentation task. 



With the above considerations, we propose a novel component, named Affinity Propagation (\texttt{APro}), 
which can be easily embedded in existing methods for label-efficient segmentation. 
Firstly, we define 
the weakly-supervised segmentation from a new perspective,  
and formulate it as a uniform affinity propagation process. The modelled pairwise term propagates the unary term to other nearby and distant pixels and updates the soft pseudo labels progressively.
Then, we introduce the global affinity propagation, which leverages the topology-aware tree-based graph and relaxes the geometric constraints of spanning tree to capture the long-range pairwise affinity. With the efficient design, the $\mathcal{O}(N^2)$ complexity of brute force implementation is reduced to $\mathcal{O}(N\log N)$, and the global propagation approach can be performed with much less resource consumption for practical applications. 
Although the long-range pairwise affinity is captured, it inevitably brings in noises based on numerous pixels in a global view.
To this end, 
we introduce a local affinity propagation to encourage the piece-wise smoothness with spatial consistency. 
The formulated \texttt{APro} can be embedded into the existing segmentation networks to generate accurate soft pseudo labels online for unlabeled regions. 
As shown in Fig.~\ref{fig:introduction}, it can be seamlessly plugged into the existing segmentation networks for various tasks to achieve weakly-supervised segmentation with label-efficient sparse annotations. 


We perform experiments on three typical label-efficient segmentation tasks, \textit{i.e.} box-supervised instance segmentation, point/scribble-supervised semantic segmentation and annotation-free semantic segmentation with pretrained CLIP model, and the results demonstrated the superior performance of our proposed universal label-efficient approach. 

\begin{figure}[t]
\begin{center}
\includegraphics[width=1.0\linewidth]{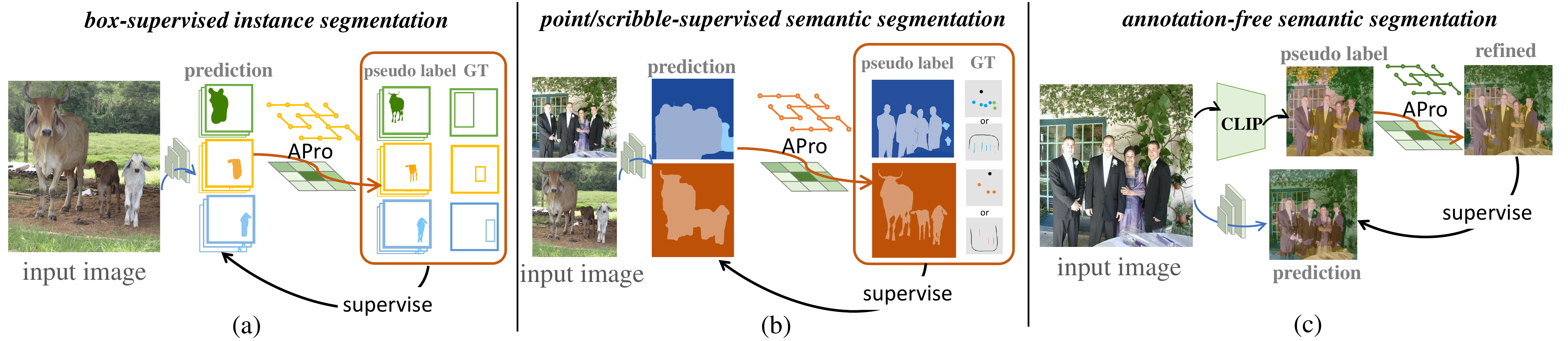} \end{center}
 \vspace{-5pt}
    \caption{The proposed approach upon the typical weakly-supervised segmentation tasks with label-efficient annotations, including (a) box-supervised instance segmentation, (b) point/scribble-supervised semantic segmentation and (c) annotation-free semantic segmentation with CLIP pretrained model.}
 \label{fig:introduction}
 \vspace{-17pt}
 \end{figure}

\section{Related Work}

\textbf{Label-efficient Segmentation.}
Label-efficient segmentation, 
which is based on the weak supervision from partial or sparse labels, has been widely explored~\cite{shen2023survey}. Different from semi-supervised settings~\cite{zhu2003semi,zhu2005semi}, this paper mainly focuses on the segmentation with sparse labels.
In earlier literature~\cite{vezhnevets2010towards,pathak2014fully,pathak2015constrained,pinheiro2015image,papandreou2015weakly},
it primarily pertained to image-level labels.
Recently, diverse sparse annotations have been employed,
including \textit{point}, \textit{scribble}, \textit{bounding box} , \textit{image-level label} and the combinations of them. 
We briefly review the weakly-supervised instance segmentation, semantic segmentation and panoptic segmentation tasks in the following.

For weakly-supervised instance segmentation, box-supervised methods~\cite{nips2019-bbtp, cvpr2021boxinst, iccv2021discobox, li2022box, li2023sim, cheng2022boxteacher, li2022box2mask,CVPR2023MAL, ICCV2023boxsnake} are dominant and perform on par with fully-supervised segmentation approaches. Besides, the ``points + bounding box'' annotation can also achieve competitive performance~\cite{cheng2022pointly,tang2022active}. 
As for weakly-supervised semantic segmentation, previous
works mainly focus on the point-level supervision~\cite{bearman2016s, chen2021seminar, tang2018regularized} and scribble-level supervision~\cite{lin2016scribblesup, zhang2021affinity, marin2019beyond}, which utilize the spatial and color information of the input image and are trained with two stages. Liang \textit{et al.}~\cite{liang2022tree} introduced an effective tree energy loss based on the low-level and high-level features for point/scribble/block-supervised semantic segmentation. For semantic segmentation, the supervision of image-level labels has been well explored~\cite{cvpr2018learning,liu2020leveraging,li2022expansion,ru2022learning}. Recently, some works have been proposed to make use of the large-scale pretrained CLIP model~\cite{radford2021learning} to achieve weakly-supervised or annotation-free semantic segmentation~\cite{lin2022clip,zhou2022extract}.
In addition, weakly-supervised panoptic segmentation methods~\cite{fan2022pointly,li2022fully, li2023point2mask} with a single and multiple points annotation have been proposed. 
In this paper, we aim to develop a universal component for various segmentation tasks, which can be easily plugged into the existing segmentation networks. 

\textbf{Pairwise Affinity Modeling.}
Pairwise affinity modeling plays an important role in many computer vision tasks.
Classical methods, like CRF~\cite{nips2011crf}, make full use of the  color and spatial information to model the pairwise relations in the semantic labeling space. Some works use it as a post-processing module~\cite{chen2017deeplab}, while others
integrate it as a jointly-trained part into the deep neural network~\cite{zheng2015conditional,obukhov2019gated}.
Inspired by CRF, some recent approaches explore the local appearance kernel to tap the neighboring pairwise affinities, where spatially nearby pixels with similar colors are more likely to fall into the same class. Tian \textit{et al.}~\cite{cvpr2021boxinst} proposed the local pairwise loss, which models the neighboring relationship in LAB color space. Some works~\cite{CVPR2023MAL, nips2019-bbtp, iccv2021discobox} focus on the pixel relation in RGB color space directly, and achieve competitive performance. 
However,  the local operation  fails to capture global context cues and  lacks the long-range affinity dependencies. Furthermore, the appearance kernel cannot reflect the topology of objects, missing the details of semantic objects. 
To model the structural pair-wise relationship, tree-based approaches~\cite{liang2022tree, nips2019treev1} leverage the minimum spanning tree~\cite{kruskal1956shortest} to capture the topology of objects in the image.
However, the affinity interactions with distant nodes  decay rapidly as the distance increases along the spanning tree.
Besides, Zhang \textit{et al.}~\cite{zhang2021affinity} proposed an affinity network to convert an image to a weighted graph, and model the node affinities by the graph neural network (GNN).
Different from these methods, in this work we model the pairwise affinity from a new affinity propagation perspective both globally and locally. 

\begin{figure}[t]
\begin{center}
\includegraphics[width=0.99\linewidth]{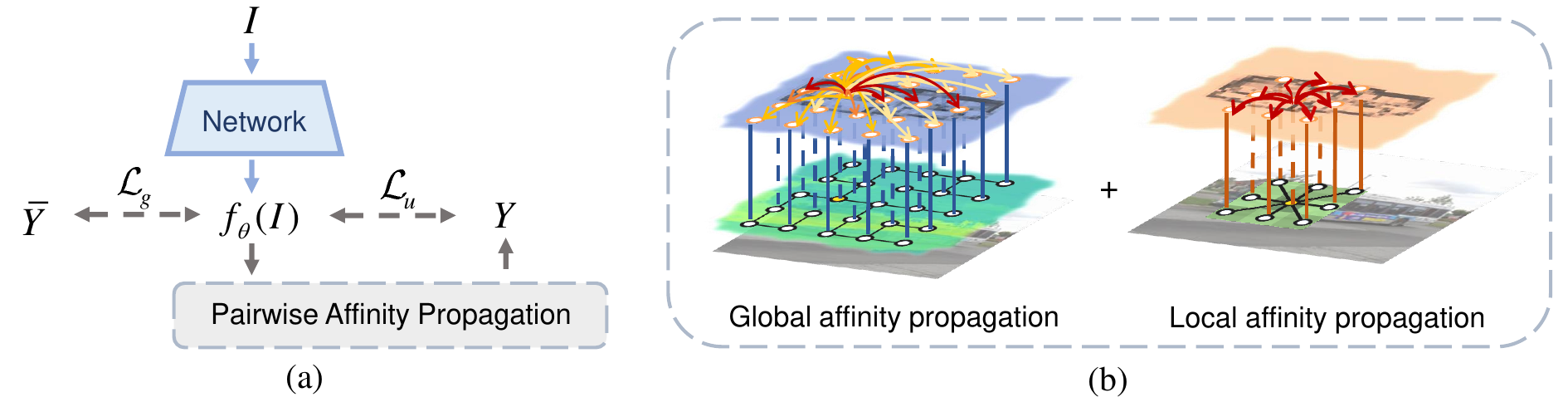} \end{center}
    \vspace{-6pt}
    \caption{Overview of our \texttt{APro} approach. (a) The general weakly supervised segmentation framework with the proposed pairwise affinity propagation. (b) The proposed approach consists of global affinity propagation (GP) and local affinity propagation (LP) to generate accurate pseudo labels.}
 \label{fig:overallnetwork}
  \vspace{-9pt}
 \end{figure}

\section{Methodology}
In this section, we introduce our proposed affinity propagation (\texttt{APro}) approach to label-efficient segmentation. First, we define the problem and model the  \texttt{APro} framework in Section~\ref{sec:problem-def}. Then, we describe the detailed pairwise affinity propagation method in Section~\ref{sec:pairwise}. Finally, we provide an efficient implementation of our method in Section~\ref{sec:efficient}.

\subsection{Problem Definition and Modeling}\label{sec:problem-def}

Given an input image $I=\{x_i\}^N$ with  $N$ pixels  and its available sparse 
ground truth labeling ${\bar Y}$ (\textit{i.e.} points,  scribble, or bounding box), and let ${f_\theta }(I)$ be the output of a segmentation network with the learnable parameters $\theta$,
the whole segmentation network can be regarded as a neural network optimization problem as follows: 
\begin{equation}
\mathop {\min }\limits_\theta  \left\{ {\mathcal{L}_{g}({f_\theta }(I), {\bar Y}) + \mathcal{L}_{u}({f_\theta }(I), Y)} \right\},
\end{equation}
where $\mathcal{L}_{g}$ is a ground truth loss on the set of labeled pixels $\Omega_g$ with ${\bar Y}$ and $\mathcal{L}_{u}$ is a loss on the unlabeled regions $\Omega_u$ with the pseudo label $Y$. 
As shown in Fig.~\ref{fig:overallnetwork}-(a),  our goal is to obtain the accurate pseudo label $Y$ by leveraging the image pixel affinities for unlabeled regions. 

As in the classic MRF/CRF model~\cite{freeman2000learning,nips2011crf}, the unary term reflects the per-pixel confidence of assigning labels, while pairwise term captures the inter-pixel constraints. We define the generation of pseudo label $Y$ as an affinity propagation process, which can be formulated as follows: 
\begin{equation}\label{equ:propagation}
{y_i} = \frac{1}{{{z_i}}}\sum\limits_{{j} \in \tau} {\phi ({x_j})} \psi ({x_i},{x_j}),
\end{equation}
where ${\phi}({x_j})$ denotes the unary potential term, which is used to align $y_j$ and the corresponding network prediction based on the available sparse labels.
$\psi ({x_i},{x_j})$ indicates the pairwise potential, which  models the inter-pixel relationships to constrain the predictions and  produce accurate pseudo labels. $\tau$ is the region with different receptive fields. $z_i$ is the summation of pairwise affinity $\psi ({x_i},{x_j})$ along with $j$ to normalize the response.

Notably, we unify both  global and local pairwise potentials in an affinity propagation process formulated in Eq.~\ref{equ:propagation}. As shown in Fig.~\ref{fig:overallnetwork}-(b),  the global affinity propagation (GP) can capture the pairwise affinities with topological consistency in a global view, while the local affinity propagation (LP) can obtain the pairwise affinities with local spatial consistency.
Through the proposed component, the soft pseudo labels $Y$ can be obtained. 
We assign each $y_i$ from GP and LP with the network prediction $p_i$ and directly employ the distance measurement function as the objective for unlabeled regions $\Omega_u$. Simple $L_1$ distance is empirically adopted in our implementation.



\subsection{Pairwise Affinity Propagation}\label{sec:pairwise}

\subsubsection{Global Affinity Propagation}


We firstly provide a solution to model the global affinity efficiently based on the input image. 
Specifically, we represent an input image as a 4-connected planar graph $\mathcal{G}$, where each node is adjacent to up to 4 neighbors. The weight of the edge measures the image pixel distance between adjacent nodes. 
Inspired by tree-based approaches~\cite{yang2014stereo,nips2019treev1}, we employ the minimum spanning tree (MST) algorithm~\cite{kruskal1956shortest} to remove the edge with a large distance to obtain the tree-based sparse graph $\mathcal{G}_T$, \textit{i.e.}, $\mathcal{G}_T \leftarrow \text{MST}(\mathcal{G})$ and $\mathcal{G}_T=  \{ \mathcal{V},\mathcal{E}\}$, where  $\mathcal{V}= \{\mathcal{V}_i\}^N$ is the set of nodes and $\mathcal{E}=\{ \mathcal{E}_i\}^{N-1}$ denotes the set of edges.

Then, we model the global pairwise potential by iterating over each node. To be specific, we take the current node as the root of the spanning tree $\mathcal{G}_T$ and propagate the long-range affinities to other nodes. While the distant nodes along the spanning tree need to pass through nearby nodes along the path of spanning tree, the distance-insensitive \texttt{max} affinity function can alleviate this geometric constraint and relax the affinity decay for long-range nodes. Hence, we define the global pairwise potential $\psi _g$ as follows:  
\begin{equation}
{\psi _g}({x_i},{x_j}) = \mathop {\mathcal{T}({I_i},{I_j})}\limits_{\forall j \in \mathcal{V}}  = \exp ( - \mathop {\max }\limits_{\forall(k,l) \in {\mathbb{E}_{i,j}}} \frac{{{w_{k,l}}}}{{{\zeta_g}^2}}),
\end{equation}
where $\mathcal{T}$ denotes the global tree. $\mathbb{E}_{i,j}$ is the set of edges in the path of $\mathcal{T}$ from node $j$ to node $i$. $w_{k,l}$ indicates the edge weight between the adjacent nodes $k$ and $l$, which is represented  as the Euclidean distance between pixel values of two adjacent nodes, $\textit{i.e.}, {w_{k,l}} = {\left| {{I_k} - {I_l}} \right|^2}$. 
$\zeta_g$ controls the degree of similarity with the long-range pixels. 
In this way, the global affinity propagation (GP) process to obtain the soft label $y^g$ can be formulated as follows: 
\begin{equation}
y_i^g = \frac{1}{{{z_i^g}}}\sum\limits_{ j \in \mathcal{V}} {\phi ({x_j})} {\psi_g}({x_i},{x_j}), \ \ {z_i^g} = \sum\limits_{j \in {\mathcal{V}}} {{\psi_g}({x_i},{x_j})},
\end{equation}
where interactions with distant nodes are performed over tree-based topology. In addition to utilizing low-level image, we empirically employ high-level feature as input to propagate semantic affinity.






\subsubsection{Local Affinity Propagation}
The long-range pairwise affinity is inevitably noisy since it is computed based on the susceptible image features in a global view. 
The spatially nearby pixels are more likely to have the same label, while they have certain difference in color and intensity, \textit{etc}. Hence, we further introduce the local affinity propagation (LP) to promote the piece-wise smoothness.
The Gaussian kernel is widely used to capture the local relationship among the neighbouring pixels in previous works~\cite{obukhov2019gated, cvpr2021boxinst, CVPR2023MAL}. 
Different from these works, we define the local pairwise affinity via the formulated affinity propagation process.
The local pairwise term $\psi _s$ is defined as: 
\begin{equation}
{\psi _s}({x_i},{x_j}) = \mathop {{\cal K}{\rm{ }}({I_i},{I_j})}\limits_{j \in {{\cal N}}(i)}  = \exp \left( {\frac{{ - {{\left| {{I_i} - {I_j}} \right|}^2}}}{{\zeta_s ^2 }}} \right),
\end{equation}
where ${\cal K}$ denotes the Gaussian kernel, $\mathcal{N}(i)$ is the set containing all local neighbor pixels. 
The degree of similarity is controlled by parameter $\zeta_s$. Then the pseudo label $y^s$ can be obtained via the following affinity propagation: 
\begin{equation}
{y^s_i} = \frac{1}{{{{z}_i^s}}}\sum\limits_{j \in {\mathcal{N}}(i)} {{\phi}({x_j})} {\psi_s}({x_i},{x_j}), \ \ {z_i^s} = \sum\limits_{j \in {\mathcal{N}}(i)} {{\psi _s}({x_i},{x_j})}, 
\end{equation}
where the local spatial consistency is maintained based on high-contrast neighbors. 
To obtain a robust segmentation performance, multiple iterations are required. Notably, our LP process ensures a fast convergence, which is 5$\times$ faster than 
MeanField-based  method~\cite{CVPR2023MAL,nips2011crf}. The details can be found in the experimental section~\ref{sec:ablation}.

\subsection{Efficient Implementation}\label{sec:efficient}
Given a tree-based graph $\mathcal{G}_T = \{ \mathcal{V},\mathcal{E}\}$ in the GP process, we define the maximum $\boldsymbol{w}$ value of the path through any two vertices as the transmission cost $\mathcal{C}$.
One straightforward approach to get $y_i^g$ of vertex $i$ is to traverse each vertex $j$ by Depth First Search or Breadth First Search to get the transmission cost $\mathcal{C}_{i,j}$ accordingly.
Consequently, the computational complexity required to acquire the entire output ${y^g}$ is $\mathcal{O}(N^2)$, making it prohibitive in real-world applications. 

Instead of calculating and updating the transmission cost of  any two vertices,
we design a lazy update algorithm to accelerate the GP process.
Initially, each node is treated as a distinct union, represented by $\mathcal{U}$. Unions are subsequently connected based on each edge $w_{k,l}$ in ascending order of $\boldsymbol{w}$. We show that when connecting two unions $\mathcal{U}_k$ and $\mathcal{U}_l$, $w_{k,l}$ is equivalent to the transmission cost for all nodes within $\mathcal{U}_k$ and $\mathcal{U}_l$. This is proved in the \textbf{supplementary material}.

To efficiently update values, we introduce a \textbf{\textit{Lazy Propagation}} scheme. We only update the value of the root node and postpone the update of its descendants. The update information is retained in a \textit{lazy tag} $\mathcal{Z}$ and is updated as follows:
\begin{equation}
\mathcal{Z}(\delta )_{k^*} = \mathcal{Z}(\delta )_{k^*}+
\begin{cases}
\text{exp}(-w_{k,l}/{\zeta_g}^2)S(\delta)_{l} & \mathcal{U}_{k}.\text{rank} > \mathcal{U}_{l}.\text{rank},\\
\text{exp}(-w_{k,l}/{\zeta_g}^2)S(\delta)_{l} - \mathcal{Z}(\delta )_{l^*}& \text{otherwise},
\end{cases}
\end{equation}
where $S(\delta)_i = \sum_{j\in \mathcal{U}_i} \delta_j$, $\delta$ means different inputs, including the dense prediction $\phi(x)$ and all-one matrix $\Lambda$. $k^*/l^*$ denotes the root node of node $k/l$.

Once all unions are connected, the lazy tags can be propagated downward from the root node to its descendants. For the descendants, the global affinity propagation term is presented as follows:
\begin{equation}
    LProp(\delta)_i = \delta_i+\sum_{r\in Asc_{\mathcal{G}_T}(i) \cup \{i\}}\mathcal{Z}(\delta)_r,
\end{equation}
 where $Asc_{\mathcal{G}_T}(i)$ represents the ascendants of node $i$ in the tree $\mathcal{G}_T$. As shown in Algorithm \ref{alg:Algorithm1}, the disjoint-set data structure is employed to implement the proposed algorithm. In our implementation, a Path Compression strategy is applied, connecting each node on the path directly to the root node. Consequently, it is sufficient to consider the node itself and its parent node to obtain $LProp$. 




\begin{algorithm}
\caption{Algorithm for GP process }\label{alg:Algorithm1}
\KwIn{Tree $\mathcal{G}_T\in \mathbb{N}^{e\times2}$; Pairwise distance $\boldsymbol{w} \in \mathbb{R}^{N}$; Dense predictions $\phi(x)\in \mathbb{R}^{N}$; \\
\hspace*{0.35 in} Vertex num $N$; Edge num $e=N-1$; Set of vertices $\mathcal{V}$.}
\KwOut{$y^g \in \mathbb{R}^{N}$. }
$\Lambda \, \leftarrow \bm{1} \in \mathbb{R}^N$\\
$F \leftarrow \{0,1,2,...,N-1\}$ \Comment{Initialize each vertex as a connected block}\\
Sort \{$\mathcal{G}_T$, $\boldsymbol{w}$\} in ascending order of $\boldsymbol{w}$. \Comment{Quick Sort}\\
\For{$(k,l) \in \mathcal{G}_T, w_i \in \boldsymbol{w}$ }{ 
$a$ $\leftarrow$ find($k$), $b$ $\leftarrow$ find($l$)\Comment{Find the root node with Path Compression}\\
Update\{$\mathcal{Z}(\phi)_a$, $\mathcal{Z}(\Lambda)_a$, $\mathcal{Z}(\phi)_b$, $\mathcal{Z}(\Lambda)_b$\}\Comment{Add lazy tag} \\
\If{$S_a<S_b$ }{
swap($a, b$) \Comment{Merge by Rank}\\
}
$F_b$ $\leftarrow$ $a$\Comment{Merge two connected blocks}\\

}
\For{$v \in \mathcal{V}$}{ 
$p \leftarrow$ find($v$) \\
\For{$\delta \in \{\phi, \Lambda\}$}{
\eIf{$p=v$}{
$LProp(\delta)_v = \mathcal{Z}(\delta)_v+\delta_v$\\
}
{
$LProp(\delta)_v = \mathcal{Z}(\delta)_p+\mathcal{Z}(\delta)_v+\delta_v $\\
}
}
$y_v^g = \frac{LProp(\phi)_v}{LProp(\Lambda)_v}$\Comment{Normalization}\\
}
\Return{$y^g$}
\end{algorithm}

\textbf{Time complexity.}
For each channel of the input, the average time complexity of sorting is $\mathcal{O}(N\log N)$.  In the merge step, we utilize the Path Compression and Union-by-Rank strategies, which have a complexity of $\mathcal{O}(\alpha(N))$\cite{tarjan1979class}. After merging all the concatenated blocks, the lazy tags can be propagated in $\mathcal{O}(N)$ time. Hence, the overall complexity is $\mathcal{O}(N\log N)$. Note that the batches and channels are independent of each other. Thus, the algorithm can be executed in parallel for both batches and channels for practical implementations.
As a result, the proposed algorithm reduces the computational complexity dramatically.

\section{Experiments}
\subsection{Weakly-supervised Instance Segmentation}\label{box-results}

\textbf{Datasets.} As in prior arts~\cite{cvpr2021boxinst,li2022box, nips2019-bbtp,iccv2021discobox}, we conduct experiments on two widely used datasets for the \textit{weakly box-supervised instance segmentation} task:
\begin{fullitemize}
\item
COCO~\cite{lin2014microsoft}, which has 80 classes with 115K \texttt{train2017} images and 5K \texttt{val2017} images.
\vspace{-2pt}
\end{fullitemize}

\begin{itemize}[leftmargin=*]
	\setlength{\itemsep}{0pt}
	\setlength{\parsep}{-2pt}
	\setlength{\parskip}{-0pt}
	\setlength{\leftmargin}{-10pt}
	\vspace{-6pt}
  \item Pascal VOC~\cite{chen2017deeplab} augmented by SBD~\cite{iccv2011SBDdataset} based on the original Pascal VOC 2012~\cite{pascalvoc2010}, which has 20 classes with 10,582 \texttt{trainaug} images and 1,449 \texttt{val} images.
  \vspace{-8pt}
\end{itemize}

\textbf{Base Architectures and Competing Methods.} In the evaluation, we apply our proposed \texttt{APro} to
two representative instance segmentation architectures, SOLOv2~\cite{nips2020solov2} and  Mask2Former~\cite{cvpr2022mask2former}, with different backbones (\textit{i.e.}, ResNet~\cite{he2016deep}, Swin-Transformer~\cite{liu2021swin}) following Box2Mask~
\cite{li2022box2mask}. 
We compare our approach  with its counterparts that model the pairwise affinity based on the image pixels without modifying the base segmentation network for box-supervised setting.
Specifically, the compared methods include Pairwise Loss~\cite{cvpr2021boxinst}, TreeEnergy Loss~\cite{liang2022tree} and CRF Loss~\cite{CVPR2023MAL}.
For fairness, we re-implement these models using the default setting in MMDetection~\cite{chen2019mmdetection}. 

\textbf{Implementation Details.}
We follow the commonly used training settings on each dataset as in MMDetection~\cite{chen2019mmdetection}. All models are initialized with ImageNet~\cite{imagenet} pretrained backbone. For SOLOv2 framework~\cite{nips2020solov2}, the scale jitter is used, where the shorter image side is randomly sampled from 640 to 800 pixels. For Mask2Former framework~\cite{cvpr2022mask2former}, the large-scale jittering augmentation scheme~\cite{ghiasi2021simple} is employed with a random scale sampled within range [0.1, 2.0], followed by a fixed size crop to 1024×1024. The initial learning rate is set to 10$^{-4}$ and the weight decay is 0.05 with 16 images per mini-batch. The box projection loss~\cite{cvpr2021boxinst, li2022box} is employed to constrain the network prediction within the bounding box label as the unary term $\phi$. 
COCO-style mask AP (\%) is adopted for evaluation. 

\textbf{Quantitative Results.} Table~\ref{tab:boxsota} shows the quantitative results. We compare the approaches with the same architecture for fair comparison. The state-of-the-art methods are listed for reference.
One can see that our \texttt{APro} method outperforms its counterparts across Pascal VOC and COCO datasets.
\begin{itemize}[leftmargin=*]
	\setlength{\itemsep}{0pt}
	\setlength{\parsep}{-2pt}
	\setlength{\parskip}{-0pt}
	\setlength{\leftmargin} 
     {-10pt}
	\vspace{-10pt}
  \item Pascal VOC~\cite{chen2017deeplab} \texttt{val}. Under the SOLOv2 framework, our approach achieves 37.1\% AP and 38.4\% AP with 12 epochs and 36 epochs, respectively, outperforming other methods by 1.4\%-2.5\% mask AP with ResNet-50. With the Mask2Former framework, our approach also outperforms its counterparts. Furthermore, with the  Swin-L backbone~\cite{liu2021swin}, our proposed approach achieves very promising performance, 49.6\% mask AP with 50 epochs.
  \vspace{-5pt}
\end{itemize}
\begin{itemize}[leftmargin=*]
	\setlength{\itemsep}{0pt}
	\setlength{\parsep}{-2pt}
	\setlength{\parskip}{-0pt}
	\setlength{\leftmargin}{-10pt}
	\vspace{-9pt}
  \item COCO~\cite{lin2014microsoft} \texttt{val}. Under the SOLOv2 framework, our approach 
  achieves 32.0\% AP and 32.9\% AP with 12 epochs and 36 epochs, and surpasses its best counterpart by 1.0\% AP and 0.4\% AP using ResNet-50, respectively. Under the Mask2Former framework, our method still achieves the best performance with ResNet-50 backbone. Furthermore, equipped with stronger backbones, our approach obtains more robust performance, achieving 38.0\% mask AP with ResNet-101, and 41.0\% mask AP using Swin-L backbone.
  \vspace{-8pt}
\end{itemize}
\textbf{Qualitative Results.} Fig.~\ref{fig:compare} illustrates the visual comparisons on affinity maps of our \texttt{APro} and other approaches, and Fig.~\ref{fig:vis} compares the segmentation results. One can clearly see that our method captures accurate  pairwise affinity with object's topology and yields more fine-grained predictions. 

\newcommand{\reshl}[2]{
\textbf{#1} \fontsize{7.5pt}{1em}\selectfont\color{mygreen}{$\!\uparrow\!$ \textbf{#2}}
}
\setlength\intextsep{0pt}
\begin{table}[t]
\setlength{\abovecaptionskip}{0cm}
\begin{center}
    \caption{Quantitative results (\S\ref{box-results}) on  Pascal VOC~\cite{chen2017deeplab} and  COCO ~\texttt{val}~\cite{lin2014microsoft} with mask AP(\%).}
    \vspace{0.1em}
    \label{tab:boxsota}
    {\hspace{-1.3ex}
    \resizebox{0.95\textwidth}{!}{
    \setlength\tabcolsep{2.2mm}
    \begin{tabular}{r||c|c||ccc||ccc}
      \hline\thickhline
      \rowcolor{mygray}
      & &  & \multicolumn{3}{c||}{\textbf{Pascal VOC} } & \multicolumn{3}{c}{\textbf{COCO}} \\
      \rowcolor{mygray}
       \multicolumn{1}{c||}{\multirow{-2}{*}{Method}}
      & \multicolumn{1}{c|}{\multirow{-2}{*}{Backbone}}  & 
        \multicolumn{1}{c||}{\multirow{-2}{*}{\#Epoch}} &
        \multicolumn{1}{c}{{AP}} & \multicolumn{1}{c}{{AP$_{\text{50}}$}} & \multicolumn{1}{c||}{{AP$_{\text{75}}$}} &  \multicolumn{1}{c}{{AP}}  & \multicolumn{1}{c}{{AP$_{\text{50}}$}} & \multicolumn{1}{c}{{AP$_{\text{75}}$}} \\\hline\hline
      {BBTP~\pub{NeurIPS19}}{~\cite{nips2019-bbtp}}                 &  {ResNet-\text{101}} &  12 & 23.1 & 54.1 & 17.1 & 21.1 & 45.5  & 17.2  \\
      \cdashline{1-9}[1pt/1pt]
      {BoxInst~\pub{CVPR21}}{~\cite{cvpr2021boxinst}}           & {ResNet-\text{50}} & 36 & 34.3 & 58.6 & 34.6 & 31.8 & 54.4 & 32.5\\
      DiscoBox~\pub{ICCV21}{~\cite{iccv2021discobox}} & {ResNet-\text{50}} & 36  & - & 59.8 & 35.5 & 31.4  & 52.6 & 32.2   \\ 
      BoxLevelset~\pub{ECCV22}{~\cite{li2022box}} & {ResNet-\text{50}} & 36 & 36.3 & 64.2 & 35.9 & 31.4 & 53.7 & 31.8  \\
      \hline
      \rowcolor{mygray1}
      \multicolumn{9}{c}{\small \em SOLOv2 Framework} \\
      \hline
      {Pairwise Loss~\pub{CVPR21}}{~\cite{{cvpr2021boxinst}}} & {ResNet-\text{50}} & 12 & 35.7 & 64.3 & 35.1 & 31.0  & 52.8 & 31.5  \\
      {TreeEnergy Loss~\pub{CVPR22}}{~\cite{{liang2022tree}}} & {ResNet-\text{50}} & 12 & 35.0 & 64.4 & 34.7 & 30.9  & 52.9 & 31.3  \\
      {CRF Loss~\pub{CVPR23}}{~\cite{{CVPR2023MAL}}} & {ResNet-\text{50}} & 12 & 35.0 & 64.7 & 34.9 & 30.9 & 53.1 & 31.4  \\
      \texttt{APro}(\texttt{Ours})                & {ResNet-\text{50}} & 12 & \textbf{37.1} & \textbf{65.1} & \textbf{37.0} & \textbf{32.0}  & \textbf{53.4} & \textbf{32.9}  \\
      \cdashline{1-9}[1pt/1pt]
      {Pairwise Loss~\pub{CVPR21}}{~\cite{{cvpr2021boxinst}}} &{{ResNet-\text{50}}} & 36 & 36.5 &  63.4 & 38.1  & 32.4 & 54.5 & 33.4  \\
      {TreeEnergy Loss~\pub{CVPR22}}{~\cite{{liang2022tree}}} & {{ResNet-\text{50}}} & 36 & 36.1 & 63.5 & 36.1  & 31.4 & 54.0 & 31.2 \\
      {CRF Loss~\pub{CVPR23}}{~\cite{{CVPR2023MAL}}} &{{ResNet-\text{50}}} & 36 & 35.9 & 64.0 & 35.7 & 32.5  & 54.9 & 33.2  \\
      \texttt{APro}(\texttt{Ours})                 &  {{ResNet-\text{50}}} & {36} & 38.4 & 65.4 & 39.8 & 32.9 & 55.2 & 33.6    \\
      \texttt{APro}(\texttt{Ours})                 & {ResNet-\text{101}}& 36 & \textbf{40.5 }& \textbf{67.9} & \textbf{42.6}  & \textbf{34.3}  & \textbf{57.0} & \textbf{35.3}   \\
      \hline
      \rowcolor{mygray1}
      \multicolumn{9}{c}{\small \em Mask2Former 
       Framework} \\
       \hline
      {Pairwise Loss~\pub{CVPR21}}{~\cite{{cvpr2021boxinst}}} & {{ResNet-\text{50}}}  & 12 &  35.2 & 62.9 & 33.9  & 33.8  & 57.1 & 34.0   \\
      {TreeEnergy Loss~\pub{CVPR22}}{~\cite{{liang2022tree}}} & {{ResNet-\text{50}}} & 12 & 36.0 & 65.0 & 34.3 &  33.5  & 56.7 & 33.7  \\
      {CRF Loss~\pub{CVPR23}}{~\cite{{CVPR2023MAL}}} & {{ResNet-\text{50}}} &  12  & 35.7 & 64.3 & 35.2 & 33.5 & 57.5 & 33.8  \\
      \texttt{APro}(\texttt{Ours})                & {{ResNet-\text{50}}}& {12}  & 37.0 & 65.1 & 37.0 & 34.4 & 57.7  & 35.3   \\
      \cdashline{1-9}[1pt/1pt]
      \texttt{APro}(\texttt{Ours})               & {{ResNet-\text{50}}}& {50} & 42.3 & 70.6 & 44.5 &  36.1 & 62.0 &  36.7   \\
      \texttt{APro}(\texttt{Ours})                 & {{ResNet-\text{101}}}& {50}  & 43.6 & 72.0 & 45.7 &  38.0 & 63.6 & 38.7   \\
       \texttt{APro}(\texttt{Ours})                 & {Swin-L}& {50} & \textbf{49.6} & \textbf{77.6} & \textbf{53.1} &\textbf{41.0} & \textbf{68.3}  & \textbf{41.9}   \\
      \hline
    \end{tabular}
    }
    }
\end{center}
\vspace{-16pt}
\end{table}

\subsection{Weakly-supervised Semantic Segmentation} \label{weak-semantic}

\textbf{Datasets.} We conduct experiments on the widely-used Pascal VOC2012 dataset~\cite{pascalvoc2010}, which contains 20 object categories and a background class. As in ~\cite{tang2018regularized,liang2022tree}, the augmented Pascal VOC dataset is adopted here. 
The \texttt{point}~\cite{bearman2016s} and \texttt{scribble}~\cite{lin2016scribblesup} annotations are employed for \textit{weakly point-supervised and scribble-supervised settings}, respectively.

\setlength\intextsep{0pt}
\begin{wraptable}{r}{0.65\linewidth}
	\centering
	\setlength{\abovecaptionskip}{0cm}
    \captionsetup{width=.65\textwidth}
    \caption{Quantitative results (\S\ref{weak-semantic}) on Pascal VOC2012~\cite{pascalvoc2010} \texttt{val} with mean IoU(\%).}
    \vspace{0.1em}
    \label{tab:sota-semantic}
    {\hspace{-1.3ex}
    \resizebox{0.65\textwidth}{!}{
    \setlength\tabcolsep{2pt}
    \renewcommand\arraystretch{1.02}
    \begin{tabular}{r||c|c|c||c}
      \hline\thickhline
      \rowcolor{mygray}
      \multicolumn{1}{c||}{Method} & \multicolumn{1}{c|}{Backbone} & \multicolumn{1}{c|}{Supervision} & \multicolumn{1}{c||}{CRF Post.}&\multicolumn{1}{c}{\textbf{mIoU}} \\\hline\hline
      $^\dagger${KernelCut Loss~\pub{ECCV18}}~\cite{tang2018regularized} & {DeepLabV2}   &  & \cmark  & 57.0 \\
      {$^\ast$TEL~\pub{CVPR22}}{~\cite{liang2022tree}}                 &        {LTF}      &   &  \xmark  & 66.8  \\
      \texttt{APro}({\texttt{Ours}}){}  & {LTF}  & \multirow{-3}{*}{{Point}} & \xmark  & \textbf{67.7}  \\
      \hline
      {$^\dagger$NormCut Loss~\pub{CVPR18}}~\cite{tang2018normalized} &{DeepLabV2} &   &  \cmark &74.5 \\
      {$^\dagger$DenseCRF Loss~\pub{ECCV18}}~\cite{tang2018regularized} & {DeepLabV2} &   & \cmark &75.0  \\
      {$^\dagger$KernelCut Loss~\pub{ECCV18}}~\cite{tang2018regularized} & {DeepLabV2} &  & \cmark & 75.0 \\
      {$^\dagger$GridCRF Loss~\pub{ICCV19}}~\cite{marin2019beyond} & {DeepLabV2} &   & 
       \xmark & 72.8  \\
      {PSI~\pub{ICCV21}}~\cite{marin2019beyond} & {DeepLabV3} &   & 
       \xmark & 74.9  \\
      {$^\ast$TEL~\pub{CVPR22}}~\cite{liang2022tree} & {LTF} &  & \xmark & 76.2 \\ 
      \texttt{APro}(\texttt{Ours})  & {LTF} & \multirow{-8}{*}{{Scribble}} & \xmark  &  \textbf{76.6}    \\  
      \hline
      \multicolumn{5}{l}{$\dagger$:adopting multi-stage training, $\ast$:our re-implementation.}
    \end{tabular}
    }
    }
\end{wraptable}
\textbf{Implementation Details.} 
As in~\cite{liang2022tree}, 
we adopt LTF~\cite{nips2019treev1} as the base segmentation model. 
The input size is 512 $\times$ 512. The SGD optimizer with momentum of 0.9 and weight decay of 10$^{-4}$ is used. The initial learning rate is  0.001, and there are \text{80}k training iterations. The same data augmentations as in~\cite{liang2022tree} are utilized. 
We employ the partial cross-entropy loss to make full use of the available point/scribble labels and constrain the unary term. ResNet-101~\cite{he2016deep} pretrained on ImageNet~\cite{imagenet} is adopted as backbone network for all methods.

\textbf{Quantitative Results.} As shown in Table~\ref{tab:sota-semantic}, we compare our \texttt{APro} approach with the state-of-the-art  methods on point-supervised and scribble-supervised semantic segmentation, respectively.

\begin{itemize}[leftmargin=*]
	\setlength{\itemsep}{0pt}
	\setlength{\parsep}{-2pt}
	\setlength{\parskip}{-0pt}
	\setlength{\leftmargin}{-10pt}
	\vspace{-6pt}
  \item Point-wise supervision. With DeepLabV2~\cite{chen2017deeplab}, KernelCut Loss~\cite{tang2018regularized} achieves 57.0\% mIoU. Equipped with LTF~\cite{nips2019treev1}, TEL~\cite{liang2022tree} achieves 66.8\% mIoU. Our \texttt{APro} achieves 67.7\% mIoU, outperforming the previous best method TEL~\cite{liang2022tree} by 0.9\% mIoU.
  \vspace{-8pt}
\end{itemize}

\begin{itemize}[leftmargin=*]
	\setlength{\itemsep}{0pt}
	\setlength{\parsep}{-2pt}
	\setlength{\parskip}{-0pt}
	\setlength{\leftmargin}{-10pt}
	\vspace{-6pt}
  \item Scribble-wise supervision. The scribble-supervised approaches are popular in weakly supervised semantic segmentation.  We apply the proposed approach under the single-stage training framework without calling for CRF post-processing during testing. Compared with the state-of-the-art methods, our approach achieves better performance with 76.6\% mIoU. 
  \vspace{-8pt}
\end{itemize}

\begin{figure}[t]
\begin{center}
\includegraphics[width=0.99\linewidth]{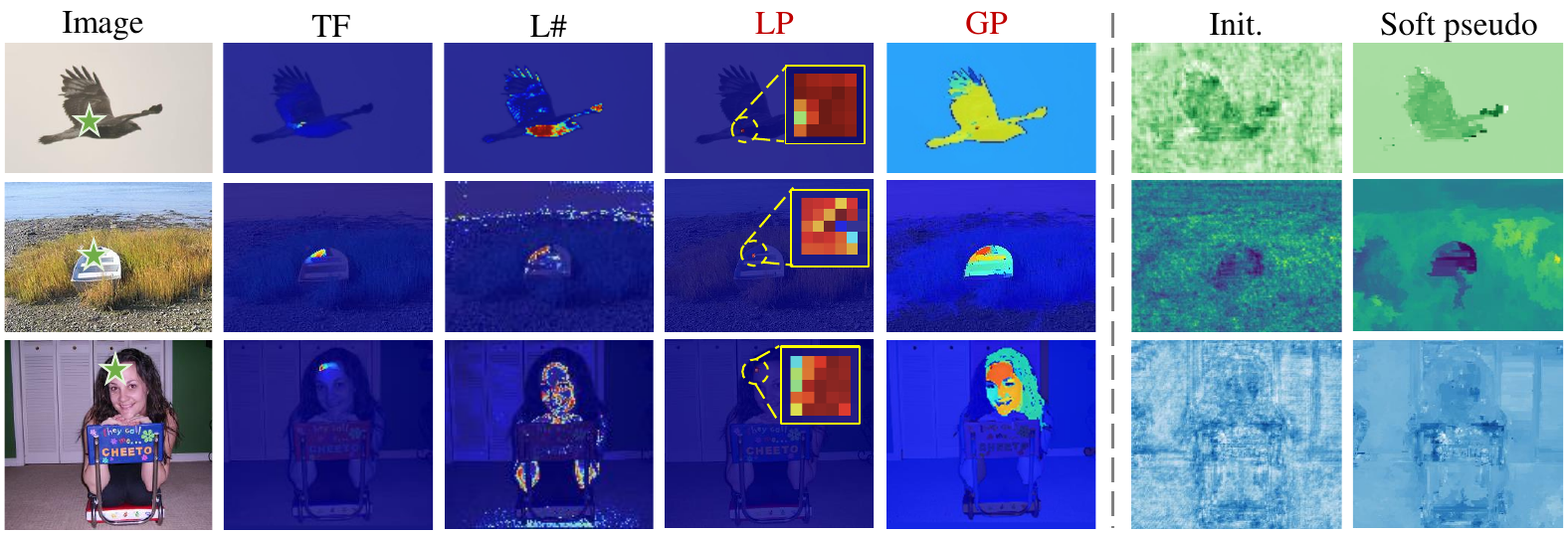} \end{center}\
    \setlength{\abovecaptionskip}{-0.3cm}
    \caption{Visual comparisons of pairwise affinity maps based on the RGB image for a specific position (green star), including TreeFilter (TF)~\cite{nips2019treev1}, local kernels with full image size (L\#), and our presented LP and GP processes. The  GP process can capture the long-range pairwise affinity with object's topology, while LP retrieves the local similarities.  Our \texttt{APro} approach smooths the noisy initial network predictions (init.) to obtain cleaner soft pseudo labels.}
    \vspace{-15pt}
 \label{fig:compare}
 \end{figure}




\subsection{CLIP-guided Semantic Segmentation}\label{clip-sup}
\textbf{Datasets.} To more comprehensively evaluate our proposed approach, we conduct experiments on \textit{CLIP-guided annotation-free semantic segmentation} with three widely used datasets:
\begin{fullitemize}
\item
Pascal VOC 2012~\cite{pascalvoc2010} introduced in Section~\ref{weak-semantic}.
\end{fullitemize}

\begin{itemize}[leftmargin=*]
	\setlength{\itemsep}{0pt}
	\setlength{\parsep}{-2pt}
	\setlength{\parskip}{-0pt}
	\setlength{\leftmargin}{-10pt}
	\vspace{-6pt}
  \item Pascal Context~\cite{mottaghi2014role}, which contains 59 foreground classes and a background class with 4,996 \texttt{train} images and 5,104 \texttt{val} images.
  \vspace{-6pt}
\end{itemize}

\begin{itemize}[leftmargin=*]
	\setlength{\itemsep}{0pt}
	\setlength{\parsep}{-2pt}
	\setlength{\parskip}{-0pt}
	\setlength{\leftmargin}{-10pt}
	\vspace{-6pt}
  \item COCO-Stuff~\cite{caesar2018coco}, which has 171 common semantic object/stuff classes on 164K images, containing 118,287 \texttt{train} images and 5,000 \texttt{val} images.
  \vspace{-8pt}
\end{itemize}



\textbf{Base Architectures and Backbones.} We employ MaskCLIP+~\cite{zhou2022extract} as our base architecture, which leverages the semantic priors of pretrained CLIP~\cite{radford2021learning} model to achieve the annotation-free dense semantic segmentation. 
In the experiments, we couple MaskCLIP+ with our \texttt{APro} approach under ResNet-50, ResNet-50$\times$16 and ViT-B/16~\cite{dosovitskiy2020vit}. 
The dense semantic predictions of MaskCLIP~\cite{zhou2022extract} are used as the unary term, and our proposed method can refine it and generate more accurate pseudo labels for training target networks. 

\textbf{Implementation Details.} For fair comparison, we keep the same settings as MaskCLIP+~\cite{zhou2022extract}. 
We keep the text encoder of CLIP unchanged and take prompts with target classes as the input. For text embedding, we feed prompt engineered texts into the text encoder of CLIP with 85 prompt templates, and average the results with the same class. 
For ViT-B/16, the bicubic interpolation is adopted for the pretrained positional embeddings. The initial learning rate is  set to 10$^{-4}$.  We train all models with batch size 32 and 2k/4k/8k iterations. DeepLabv2-ResNet101 is used as the backbone.

\setlength\intextsep{0pt}
\begin{wraptable}{r}{0.68\linewidth}
	\centering
	\setlength{\abovecaptionskip}{0cm}
    \captionsetup{width=.65\textwidth}
    \caption{Quantitative results (\S\ref{clip-sup}) on Pascal VOC2012~\cite{pascalvoc2010}~\texttt{val}, Pascal Context~\cite{mottaghi2014role} \texttt{val}, and COCO-Stuff~\cite{caesar2018coco} \texttt{val} with mean IoU (\%).}
    \vspace{0.1em}
    \label{tab:maskclip-sota}
    {\hspace{-1.3ex}
    \resizebox{0.65\textwidth}{!}{
    \setlength\tabcolsep{2pt}
    \renewcommand\arraystretch{1.02}
    \begin{tabular}{r||c||c||c||c}
      \hline\thickhline
      \rowcolor{mygray}
      \multicolumn{1}{c||}{Method} & \multicolumn{1}{c||}{CLIP Model} & \multicolumn{1}{c||}{ \textbf{VOC2012} } & \multicolumn{1}{c||}{ \ \textbf{Context} } & \multicolumn{1}{c}{\!\!\!\textbf{COCO.}\!\!\!} \\\hline\hline
      {MaskCLIP+~\pub{ECCV22}}{~\cite{zhou2022extract}}  &  &  58.0  &  23.9 &  13.6   \\
      \texttt{APro}(\texttt{Ours})                & \multirow{-2}{*}{{ResNet-$\text{50}$}} &  \reshl{61.6}{3.6} &   \reshl{25.4}{1.5} &   \reshl{14.6}{1.0}  \\
      \cdashline{1-5}[1pt/1pt]
      {MaskCLIP+~\pub{ECCV22}}{~\cite{zhou2022extract}}  &  &  67.5  &  25.2 &  17.3   \\
      \texttt{APro}(\texttt{Ours})                 & \multirow{-2}{*}{{ResNet-$\text{50}$$\times$16}} &  \reshl{70.4}{2.9} & \reshl{26.5}{1.3}  &   \reshl{18.2}{0.9}  \\
      \cdashline{1-5}[1pt/1pt]
      {MaskCLIP+~\pub{ECCV22}}{~\cite{zhou2022extract}}           &                                                             & 73.6 & 31.1 & 18.0 \\
      \texttt{APro}(\texttt{Ours})                 & \multirow{-2}{*}{{ViT-B/16}} & \reshl{75.1}{1.5} & \reshl{32.6}{1.5} & \reshl{19.5}{1.5} \\
      \hline
    \end{tabular}
    }
    }
\end{wraptable}
\textbf{Quantitative Results.} Table~\ref{tab:maskclip-sota} compares our approach with MaskCLIP+~\cite{zhou2022extract} for annotaion-free semantic segmentation. We have the following observations. 
\begin{itemize}[leftmargin=*]
	\setlength{\itemsep}{0pt}
	\setlength{\parsep}{-2pt}
	\setlength{\parskip}{-0pt}
	\setlength{\leftmargin}{-10pt}
	\vspace{-6pt}
  \item Pascal VOC2012~\cite{pascalvoc2010} \texttt{val}. With ResNet-50 as the image encoder in pretrained CLIP model, our approach outperforms MaskCLIP+ by 3.6\% mIoU. With ResNet-50$\times$16 and ViT-B/16 as the image encoders, our method surpasses MaskCLIP+ by 2.9\% and 1.5\% mIoU, respectively. 
  \vspace{-8pt}
\end{itemize}

\begin{itemize}[leftmargin=*]
	\setlength{\itemsep}{0pt}
	\setlength{\parsep}{-2pt}
	\setlength{\parskip}{-0pt}
	\setlength{\leftmargin}{-10pt}
	\vspace{-6pt}
  \item Pascal Context~\cite{mottaghi2014role} 
  \texttt{val}. Our proposed method outperforms MaskCLIP+ consistently with different image encoders (about +1.5\% mIoU).
  \vspace{-8pt}
\end{itemize}

\begin{itemize}[leftmargin=*]
	\setlength{\itemsep}{0pt}
	\setlength{\parsep}{-2pt}
	\setlength{\parskip}{-0pt}
	\setlength{\leftmargin}{-10pt}
	\vspace{-6pt}
  \item COCO-Stuff~\cite{caesar2018coco} \texttt{val}. COCO-Stuff consists of hundreds of semantic categories. Our method still brings +1.0\%, +0.9\% and +1.5\% performance gains over MaskCLIP+ with ResNet-50, ResNet-50$\times$16 and ViT-B/16 image encoders, respectively. 
  \vspace{-8pt}
\end{itemize}

\begin{figure}[t]
\begin{center}
\includegraphics[width=0.99\linewidth]{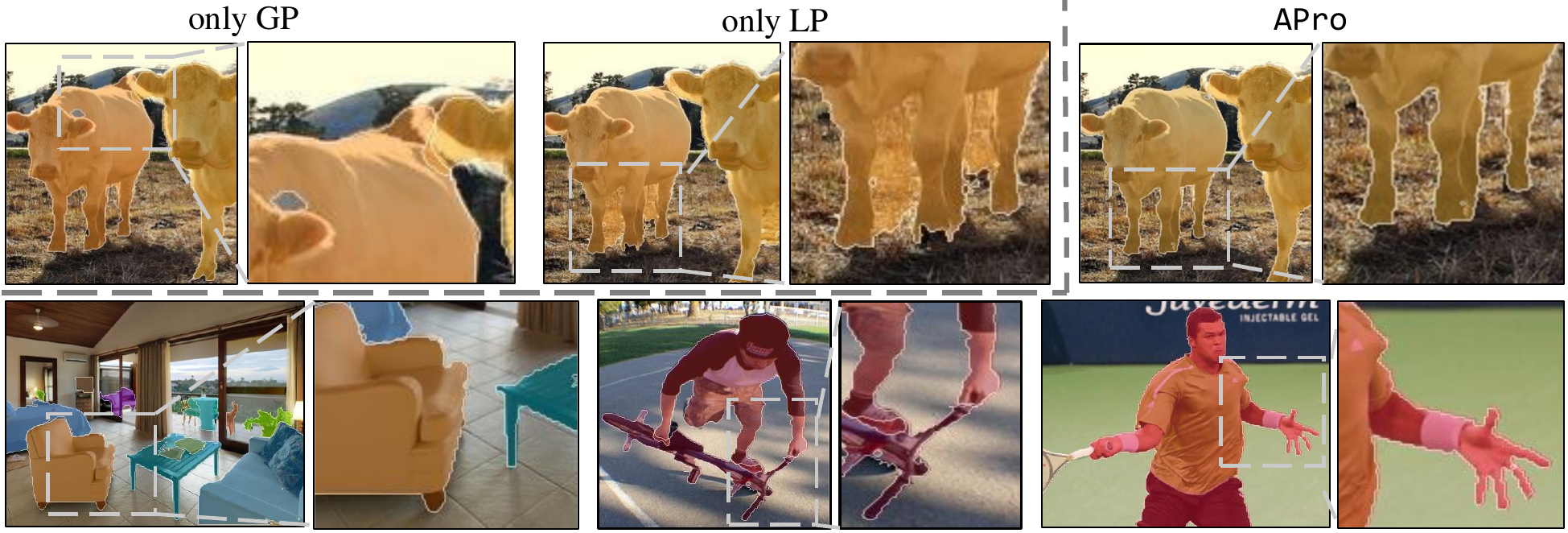} \end{center}
   \vspace{-9pt}
    \caption{Qualitative results with different pairwise affinity terms  on weakly supervised instance segmentation. Our method only with GP process preserves details without local consistency, while the model with only LP process encourages the local smoothness without topology-wise details. Our \texttt{APro} approach yields high-quality predictions with fine-grained details.}
 \label{fig:vis}
 \vspace{-14pt}
 \end{figure}

\subsection{Diagnostic Experiments}\label{sec:ablation}

For in-depth analysis, we conduct ablation studies on Pascal VOC~\cite{pascalvoc2010} upon the weakly box-supervised instance segmentation task. 

\begin{wraptable}{r}{0.52\linewidth}
	\centering
 \vspace{-0.2em}
	\setlength{\abovecaptionskip}{0cm}
    \captionsetup{width=.52\textwidth}
    \caption{Effects of unary and pairwise terms.}
    \label{tab:unary_pairwise}
    {\hspace{-1.5ex}
    \resizebox{0.52\textwidth}{!}{
    \setlength\tabcolsep{4pt}
    \renewcommand\arraystretch{1.0}
    \begin{tabular}{ccc||ccc}
        \hline\thickhline
        \rowcolor{mygray}
           Unary & Global Pairwise & Local  Pairwise & AP & AP$_{50}$ & AP$_{75}$  \\ 
        \hline\hline
        \cmark &  & &  25.9  & 57.0 & 20.4  \\
        \cmark  & \cmark & & 36.3 & 63.9 & 37.0  \\
        \cmark  &       &\cmark &  36.0 & 64.3 & 35.6 \\
         \cmark  &  \cmark     &\cmark & \textbf{38.4}  & \textbf{65.4} & \textbf{39.8}  \\
        \hline
    \end{tabular}
    }
    }
\end{wraptable}
\textbf{Unary and Pairwise Terms.}
Table~\ref{tab:unary_pairwise} shows the evaluation results with different unary and pairwise terms. 
When using the unary term only, our method achieves 25.9\% AP. 
When the global pairwise term is employed, our method achieves a much better performance of 36.3\% AP. Using the local pairwise term only, our method obtains 36.0\% AP.
When both the global and local pairwise terms are adopted, our method achieves the best performance of 38.4\% AP.


\begin{wraptable}{r}{0.50\linewidth}
	\centering
 \vspace{-0.2em}
	\setlength{\abovecaptionskip}{0cm}
    \captionsetup{width=.50\textwidth}
    \caption{Comparisons with tree-based methods.}
    \label{tab:tree}
    {\hspace{-1.5ex}
    \resizebox{0.50\textwidth}{!}{
    \setlength\tabcolsep{4pt}
    \renewcommand\arraystretch{1.0}
    \begin{tabular}{c||ccc}
        \hline\thickhline
        \rowcolor{mygray}
        Method &  AP & AP$_{50}$ & AP$_{75}$ \\ 
        \hline\hline
       TreeFilter~\cite{nips2019treev1} & 36.1 & 63.5 & 36.1  \\
       TreeFilter~\cite{nips2019treev1} + Local Pairwise  & 36.8 & 64.4 & 36.5  \\
       Global + Local Pairwise (\texttt{Ours}) &  \textbf{38.4} & \textbf{65.4} & \textbf{39.8} \\
        \hline
    \end{tabular}
    }
    }
\end{wraptable}
\textbf{Tree-based Long-range Affinity Modeling.}
The previous works~\cite{nips2019treev1,liang2022tree} explore tree-based filters for pairwise relationship modeling.  Table~\ref{tab:tree} compares our method with them.
TreeFilter can capture the relationship with distant nodes to a certain extent (see Fig.~\ref{fig:compare}). Directly using TreeFilter as the pairwise term leads 
to 36.1\% AP.  By combining TreeFilter with our local pairwise term, the model obtains 36.8\%AP. 
In comparison, our proposed approach achieves 38.4\% AP.

\begin{wraptable}{r}{0.32\linewidth}
	\centering
   \vspace{-0.2em}
	\setlength{\abovecaptionskip}{0cm}
    \captionsetup{width=.32\textwidth}
    \caption{Comparisons on local pairwise affinity modeling.}
    \label{tab:local}
    {\hspace{-1.5ex}
    \resizebox{0.32\textwidth}{!}{
    \setlength\tabcolsep{4pt}
    \renewcommand\arraystretch{1.0}
    \begin{tabular}{cc||cc}
        \hline\thickhline
        \rowcolor{mygray}
        \multicolumn{2}{c||}{LP(\texttt{Ours})} & \multicolumn{2}{c} {MeanField\cite{CVPR2023MAL}} \\
        \rowcolor{mygray}
          Iteration  & AP  & Iteration & AP \\
        \hline\hline
          10  & 35.8  &  20 & 35.2 \\ 
          20  & \textbf{36.0}  & 30  & 35.5 \\ 
          30  & 35.7  &  50  & 35.5  \\
          50  & 35.6 & 100 & \textbf{35.9} \\
        \hline
    \end{tabular}
    }
    }
\end{wraptable}
\textbf{Iterated Local Affinity Modeling.}
We evaluate our local affinity propagation (LP) with different iterations, and compare it with the classical MeanField  method~\cite{CVPR2023MAL,nips2011crf}.  Table~\ref{tab:local} reports the comparison results.
Our \texttt{APro} with the LP process achieves 36.0\% AP after 20 iterations. However, replacing our local affinity propagation with MeanFiled-based method~\cite{CVPR2023MAL} costs 100 iterations to obtain 35.9\% AP. This indicates that our LP method possesses the attribute of fast convergence.

\begin{wraptable}{r}{0.31\linewidth}
	\centering
        \vspace{-0.2em}
	\setlength{\abovecaptionskip}{0cm}
    \captionsetup{width=.31\textwidth}
    \caption{Generation of soft pseudo labels.}
    \label{tab:manner}
    {\hspace{-1.5ex}
    \resizebox{0.31\textwidth}{!}{
    \setlength\tabcolsep{4pt}
    \renewcommand\arraystretch{1.0}
    \begin{tabular}{c||ccc}
        \hline\thickhline
        \rowcolor{mygray}
        Method &  AP & AP$_{50}$ & AP$_{75}$ \\ 
        \hline\hline
       GP-LP-C & 36.8  & 63.7 & 37.8  \\
       LP-GP-C  & 37.7 & 65.1 & 39.1  \\
       GP-LP-P &  \textbf{38.4} & \textbf{65.4} & \textbf{39.8} \\
        \hline
    \end{tabular}
    }
    }
\end{wraptable}
\textbf{Soft Pseudo-label Generation.} 
With the formulated GP and LP methods, we study how to integrate them to generate the soft pseudo-labels in Table~\ref{tab:manner}. 
We can cascade GP and LP sequentially  to refine the pseudo labels. Putting GP before LP  (denoted as GP-LP-C) achieves  36.8\% AP, and putting LP before GP (denoted as LP-GP-C) performs better with 37.7\% AP. 
In addition, we can use GP and LP in parallel (denoted as
GP-LP-P) to produce two pseudo labels, and employ both of them to optimize the segmentation network with $L_1$ distance.
Notably, GP-LP-P achieves the best performance with 38.4\% mask AP. 
This indicates that our proposed affinity propagation in global and local views are complementary for optimizing the segmentation network.

\begin{wraptable}{r}{0.28\linewidth}
	\centering
 \vspace{-0.2em}
	\setlength{\abovecaptionskip}{0cm}
    \captionsetup{width=.28\textwidth}
    \caption{Average runtime (ms) with and without the efficient implementation.}
    \label{tab:runtime}
    {\hspace{-1.5ex}
    \resizebox{0.28\textwidth}{!}{
    \setlength\tabcolsep{4pt}
    \renewcommand\arraystretch{1.0}
    \begin{tabular}{c||c}
        \hline\thickhline
        \rowcolor{mygray}
           Effic. Imple. & Ave. Runtime \\
        \hline\hline
        \xmark  &  4.3$\times$10$^3$  \\
        \cmark &  0.8 \\
        \hline
    \end{tabular}
    }
    }
\end{wraptable}
\textbf{Runtime Analysis.} We report the average runtime of our method in Table~\ref{tab:runtime}.
The experiment is conducted on a single GeForce RTX 3090 with batch size 1. Here we report the average runtime for one GP process duration of an epoch on the Pascal VOC dataset.
When directly using Breadth First Search for each node with $N$ times, the runtime is 4.3$\times$10$^3$  ms with $\mathcal{O}(N^2)$ time complexity. While employing the proposed efficient implementation, the runtime is only 0.8 ms with $\mathcal{O}(N\log N)$ time complexity. This demonstrates that the proposed efficient implementation reduces the computational complexity dramatically.

\section{Conclusion}
In this work, we proposed a novel universal component for weakly-supervised segmentation by formulating it as an affinity propagation process. A global and a local pairwise affinity term were introduced to generate the accurate soft pseudo labels. An efficient implementation with the light computational overhead was developed.
The proposed approach, termed as \texttt{APro}, can be embedded into the existing segmentation networks for label-efficient segmentation. Experiments on three typical label-efficient segmentation tasks, \textit{i.e.}, box-supervised instance segmentation, point/scribble-supervised semantic segmentation and CLIP-guided annotation-free semantic segmentation, proved the effectiveness of proposed method.

\section*{Acknowledgments}
This work is supported by National Natural Science Foundation of China under Grants (61831015). It is also supported by the Information Technology Center and State Key Lab of CAD\&CG, Zhejiang University.

{\small
\bibliographystyle{unsrt}
\bibliography{reference}
}

\clearpage
\appendix
\section*{Supplementary Material}

\setcounter{figure}{0}
\setcounter{table}{0}

\renewcommand{\thefigure}{A\arabic{figure}}
\renewcommand{\thetable}{A\arabic{table}}

In this document, we provide more details, additional experimental results and discussions on our approach. The supplementary material is organized as follows:

\begin{itemize}[leftmargin=*]
	\setlength{\itemsep}{0pt}
	\setlength{\parsep}{-0pt}
	\setlength{\parskip}{-0pt}
	\setlength{\leftmargin}{-10pt}
	\vspace{-2pt}
  \item \S\ref{sec:proof}: more details on the efficient implementation;
  \item \S\ref{sec:graphical}: additional graphical illustration; 
  \item \S\ref{sec:comparison}: more performance comparisons;
  \item \S\ref{sec:vis}: additional visualization results;
  \item \S\ref{sec:discussion}: discussions.
  \vspace{-4pt}
\end{itemize}

\section{More Details on the Efficient Implementation.}\label{sec:proof}
In this section, we first present the proofs of our claims about transmission cost and lazy propagation in our proposed lazy update algorithm. Then, we provide the pseudo-code of the \textit{find} function in Algorithm 1 of the main paper. The symbols in this document follow the same definitions as the main paper.

\subsection{Proofs on Transmission Cost and Lazy Propagation}


\newtheorem{lemma}{Lemma}
\begin{lemma}
Given edge $\mathcal{E}_{(k,l)}$ in $\mathcal{G}_T$ with edge weight $w_{k,l}$, $\forall a \in \mathcal{U}_{k}$, $b \in \mathcal{U}_{l}$, the transmission cost between vertex a and b is $w_{k,l}$.
\end{lemma}

\begin{proof}

Since there are no  loops in the tree, the shortest path between any two vertices is unique. Therefore, there exists a path $a{-}k$ in  $\mathcal{U}_k$ that connects vertices $a$ and $k$, and a path $b{-}l$ that connects $b$ and $l$ in  $\mathcal{U}_l$. When connecting unions $\mathcal{U}_k$ and $\mathcal{U}_l$ through edge $\mathcal{E}_{(k,l)}$, 
there is exactly a single path connecting $a$ and $b$, denoted as $a{-}k{-}l{-}b$. As the weight $\boldsymbol{w}$ is sorted in ascending order, for any edge $\mathcal{E}_i$ with $w_i$ between $a{-}k$ in $\mathcal{U}_k$, we have $w_i \leq w_{k,l}$.
The same conclusion applies to $l{-}b$. Hence, the maximum weight in path $a{-}k{-}l{-}b$ is $w_{k,l}$.
Consequently, once $k$ and $l$ are connected, $w_{k,l}$ is equivalent to the transmission cost for all nodes within $\mathcal{U}_k$ and $\mathcal{U}_l$.

\end{proof}

\begin{lemma}
When connecting vertices $k$ and $l$, lazy tags $\mathcal{Z}(\delta )_{k^*}$ and $\mathcal{Z}(\delta )_{l^*}$ can be updated as follows:

\begin{equation}
\mathcal{Z}(\delta )_{k^*} = \mathcal{Z}(\delta )_{k^*}+
\begin{cases}
\text{exp}(-w_{k,l}/{\zeta_g}^2)S(\delta)_{l} & \mathcal{U}_{k}.\text{rank} > \mathcal{U}_{l}.\text{rank},\\
\text{exp}(-w_{k,l}/{\zeta_g}^2)S(\delta)_{l} - \mathcal{Z}(\delta )_{l^*}& \text{otherwise}.
\end{cases}
\end{equation}

\end{lemma}

\begin{proof}
Given $ a\in \mathcal{U}_k$, for $\forall b \in \mathcal{U}_l$, the transmission cost between $a$ and $b$ is $w_{k,l}$. We have:

\begin{equation}
    \Delta LProp(\delta)_a = \sum_{i \in \mathcal{U}_l} ( \text{exp}(-w_{k,l}/{\zeta_g}^2)\delta_i )= \text{exp}(-w_{k,l}/{\zeta_g}^2)S(\delta)_l.
\end{equation}

First, let $\mathcal{U}_{k}.\text{rank} > \mathcal{U}_{l}.\text{rank}$. When merging unions $\mathcal{U}_k$ and $\mathcal{U}_l$, we choose $k^*$ as the root node and let $l^*$ be its descendant. There is: 

\begin{equation}
    \Delta LProp(\delta)_a = \Delta \mathcal{Z}(\delta)_{k^*},
\end{equation}

\begin{equation}
    \therefore \Delta \mathcal{Z}(\delta)_{k^*} =  \text{exp}(-w_{k,l}/{\zeta_g}^2)S(\delta)_l.
\label{eq:tzm}
\end{equation}

Second, let $\mathcal{U}_{k}.\text{rank} \leq \mathcal{U}_{l}.\text{rank}$. When merging unions $\mathcal{U}_k$ and $\mathcal{U}_l$, we instead choose $l^*$ as the root node and let $k^*$ be its descendant. Then we have:

\begin{equation}
    \Delta LProp(\delta)_a = \mathcal{Z}(\delta)_{l^*} + \Delta \mathcal{Z}(\delta)_{k^*}  = \text{exp}(-w_{k,l}/{\zeta_g}^2)S(\delta)_l,
\end{equation}

\begin{equation}
    \therefore \Delta \mathcal{Z}(\delta)_{k^*} =  \text{exp}(-w_{k,l}/{\zeta_g}^2)S(\delta)_l - \mathcal{Z}(\delta)_{l^*}.
\end{equation}

\end{proof}

\subsection{Pseudo Code}
The pseudo-code of the \textit{find} function is shown in Algorithm~\ref{alg:find}, which finds the root rode with Path Compression.

\begin{algorithm}
\renewcommand{\thealgocf}{A1}
    \setstretch{1.1}
\caption{Pseudo-code of the \textit{find} function with Path Compression}
\label{alg:find}
\definecolor{codeblue}{rgb}{0.25,0.5,0.5}
\lstset{
    backgroundcolor=\color{white},
    basicstyle=\fontsize{8.5pt}{8pt}\ttfamily\selectfont,
    columns=fullflexible,
    breaklines=true,
    captionpos=b,
    escapeinside={(:}{:)},
    commentstyle=\fontsize{8.5pt}{8pt}\color{codeblue},
    keywordstyle=\fontsize{8.5pt}{8.5pt},
}
\begin{lstlisting}[language=C]
/*
fa: the parent of the i-th node, shape: (N)
tag: the lazy tag of numbers
ptag: the lazy tag of predictions
*/

(:\color{codetype}{\textbf{int}}:) (:\color{codefunc}{\textbf{find}}:)((:\color{codetype}{\textbf{int}}:) x){
/*
x: the node index to query
return: the root node of x
*/
    (:\color{codetype}{\textbf{int}}:) fx = (:\color{codevar}{\textbf{fa}}:)[x];
    (:\color{codedefine}{\textbf{if}}:)(fx == x)
        (:\color{codedefine}{\textbf{return}}:) x;
    
    (:\color{codevar}{\textbf{fa}}:)[x] = (:\color{codefunc}{\textbf{find}}:)(fx);      // Path Compression
    (:\color{codedefine}{\textbf{if}}:)((:\color{codevar}{\textbf{fa}}:)[x] != fx){
        (:\color{codevar}{\textbf{tag}}:)[x] += (:\color{codevar}{\textbf{tag}}:)[fx]; // Downlink lazy tag
        (:\color{codevar}{\textbf{ptag}}:)[x] += (:\color{codevar}{\textbf{ptag}}:)[fx];
    }
        
    (:\color{codedefine}{\textbf{return}}:) (:\color{codevar}{\textbf{fa}}:)[x];
}

\end{lstlisting}
\end{algorithm}



\section{Additional Graphical Illustration}\label{sec:graphical}
To facilitate a better comprehension, we provide a detailed graphical illustration in Fig.~\ref{fig:graph} to describe our global affinity propagation process. Initially, an input image is represented as a 4-connected planar graph. Subsequently, the Minimum Spanning Tree (MST) is constructed based on the edge weights to obtain the tree-based graph $\mathcal{G}_T$. ${\psi_g}({x_i},{x_j})$ is calculated as $exp(-d)$, where $d$ is the maximum value along the path $E_{i,j}$ from node $x_i$ to node $x_j$. 
This pairwise similarity ${\psi_g}({x_i},{x_j})$ is then multiplied by the unary term to obtain soft pseudo predictions.

Note that Fig.~\ref{fig:graph} serves purely as a visual illustration of our method. In the implementation, it is unnecessary to compute as it explicitly. As detailed in Section 3.3 of main paper, we alternatively design a lazy propagation scheme to efficiently update these values.

\begin{figure}[H]
\begin{center}
\includegraphics[width=0.95\linewidth]{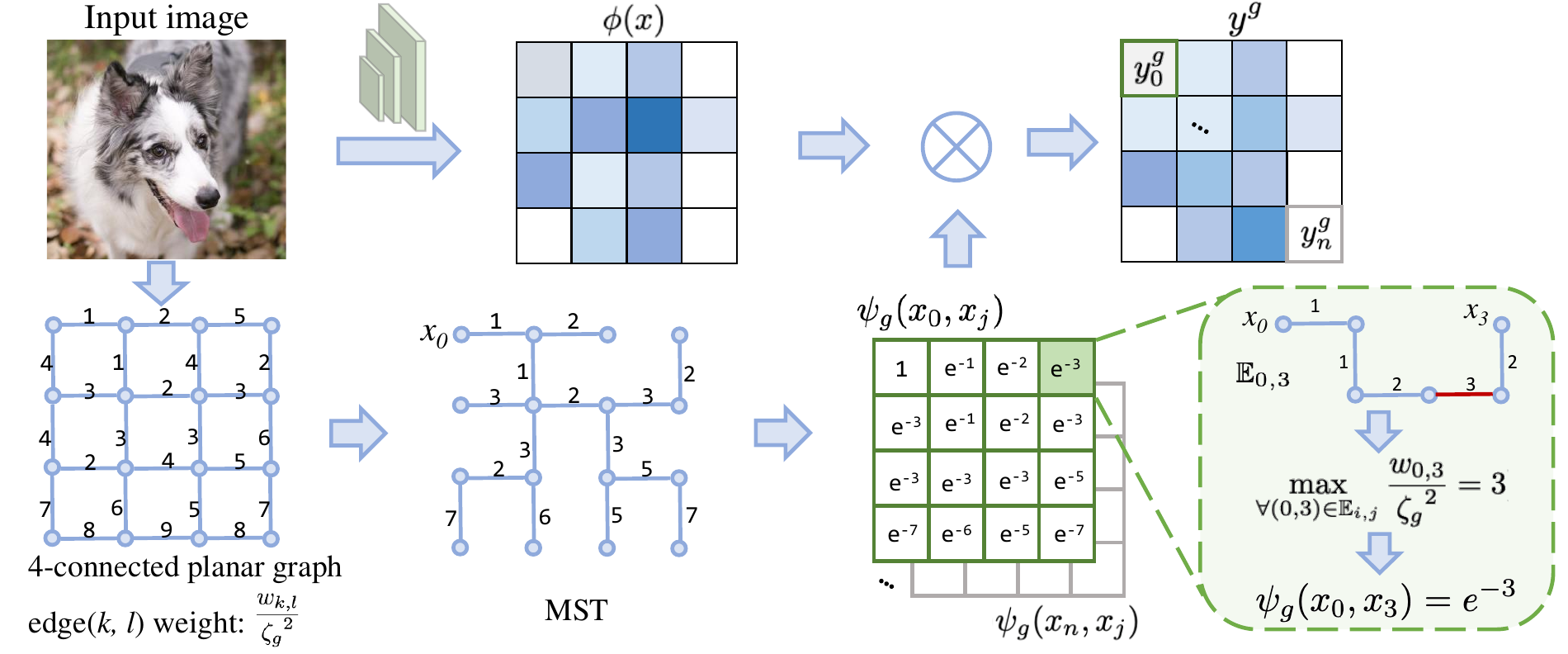} \end{center}
     \vspace{-5pt}
    \caption{The graphical illustration of the detailed process of global affinity propagation. In the green dashed box, we present the calculation of $\psi_g(x_0,x_3)$ as a simple example.}
 \label{fig:graph}
 \vspace{-14pt}
 \end{figure}

\section{More Performance Comparisons}\label{sec:comparison}
For annotation-free semantic segmentation with pretrained CLIP model, 
Key Smoothing (KS) proposed in MaskCLIP~\cite{zhou2022extract} also aims to realize the global affinity propagation. 
To better explore their efforts, we conduct detailed comparisons between KS and our \texttt{APro} method based on training-free MaskCLIP~\cite{zhou2022extract}. The experimental results are shown in Table~\ref{tab:maskclip-sota}.

\setlength\intextsep{0pt}
\begin{wraptable}{r}{0.52\linewidth}
	\centering
	\setlength{\abovecaptionskip}{0cm}
    \captionsetup{width=.52\textwidth}
    \caption{Quantitative results on Pascal Context~\cite{mottaghi2014role} \texttt{val} and COCO-Stuff~\cite{caesar2018coco} \texttt{val} with mean IoU (\%).}
    \vspace{0.1em}
    \label{tab:maskclip-sota}
    {\hspace{-1.3ex}
    \resizebox{0.52\textwidth}{!}{
    \setlength\tabcolsep{2pt}
    \renewcommand\arraystretch{1.02}
    \begin{tabular}{c||c||c||c}
      \hline\thickhline
      \rowcolor{mygray}
      \multicolumn{1}{c||}{Method} & \multicolumn{1}{c||}{CLIP Model} & \multicolumn{1}{c||}{ \ \textbf{Context} } & \multicolumn{1}{c}{ \textbf{COCO.}} \\\hline\hline
      MaskCLIP~\cite{zhou2022extract} &   &  18.46 &  10.17   \\
      +KS  &  & 21.0 & 12.42 \\
      +\texttt{APro}(\texttt{Ours})                & \multirow{-3}{*}{{ResNet-$\text{50}$}} & \textbf{21.67} &  \textbf{12.70}  \\
      \cdashline{1-4}[1pt/1pt]
      MaskCLIP~\cite{zhou2022extract}  &   &  21.57 &  13.55 \\
      +KS  &  & 22.65 & 15.50 \\
      +\texttt{APro}(\texttt{Ours})                 & \multirow{-3}{*}{{ResNet-$\text{50}$$\times$16}} &  \textbf{24.03}  & \textbf{16.30}  \\
      \cdashline{1-4}[1pt/1pt]
     MaskCLIP~\cite{zhou2022extract}           &                                                       & 21.68 & 12.51 \\
     +KS  &  & 23.87 & 13.79 \\
     +KS+PD  &  & 25.45 & 14.62 \\
     +\texttt{APro}(\texttt{Ours})  &  & 28.91 & 16.69 \\
     +\texttt{APro}(\texttt{Ours}) +PD                & \multirow{-5}{*}{{ViT-B/16}} & \textbf{29.42} & \textbf{16.71} \\
      \hline
    \end{tabular}
    }
    }
\end{wraptable}

Both KS and our \texttt{APro} method bring performance gains. Compared with KS, \texttt{APro} achieves better performance with different CLIP-based models. Especially, for ViT-B/16 model, our approach outperforms KS by +5.04\% mIoU on Pascal Context and +2.90\% mIoU on COCO, repectively. Equipped with Prompt Denoising (PD), the models could achieve further improvements.

We have the following further discussions: KS relies on the calculation of key feature similarities, which predominantly stems from high-level features of CLIP and computes pairwise terms within each pair of patches. Compared with KS of MaskCLIP, our method is built on a tree-based graph derived from low-level images, which is capable of capturing finer topological details. 

\section{Additional Visualization Results}\label{sec:vis}
To further show the performance of our proposed \texttt{APro} approach, we provide more visualization results.
Fig.~\ref{fig:box_voc} shows the qualitative comparisons with the state-of-the-art methods upon box-supervised instance segmentation task~\cite{cvpr2021boxinst,li2022box,li2022box2mask}.
It can be seen that our proposed \texttt{APro} approach is able to generate more accurate boundaries. 
For weakly-supervised semantic segmentation, we compare our method with the prior art TEL~\cite{liang2022tree} upon point-wise supervision in Fig.~\ref{fig:point_sem}. \texttt{APro} captures the fine-grained details of objects with the fitting boundaries. 
As for CLIP-guided annotation-free semantic segmentation, Fig.~\ref{fig:clip-guide} provides the comparison results with MaskCLIP+~\cite{zhou2022extract}. 
It can be observed that our approach eliminates the noisy predictions from the pretrained CLIP model effectively, achieving high-quality mask predictions. In addition, Fig.~\ref{fig:vis_coco} provides the qualitative results of our method on general COCO dataset.


\section{Discussions}\label{sec:discussion}
\textbf{Asset License and Consent.}
We use four image segmentation datasets, \textit{i.e.}, COCO~\cite{lin2014microsoft}, Pascal VOC 2012~\cite{pascalvoc2010}, COCO-Stuff~\cite{caesar2018coco} and Pascal Context~\cite{mottaghi2014role}, which are all publicly and freely available for academic research. 
We implement all models with MMDetection~\cite{chen2019mmdetection}, MMSegmentation~\cite{mmseg2020} and openseg.pytorch~\cite{openseg} codebases.
COCO (\href{https://cocodataset.org/}{\texttt{https://cocodataset.org/}}) is released under the
\href{https://creativecommons.org/licenses/by/4.0/legalcode}{CC BY 4.0}.
Pascal VOC 2012 (\href{http://host.robots.ox.ac.uk/pascal/VOC/voc2012/}{\texttt{http://host.robots.ox.ac.uk/pascal/VOC/voc2012/}}) is released under the \href{https://www.flickr.com/creativecommons/}{Flickr Terms of use}
for images.
COCO-Stuff v1.1 (\href{https://github.com/nightrome/cocostuff}{\texttt{https://github.com/nightrome/cocostuff}}) is released under the \href{https://www.flickr.com/creativecommons/}{Flickr Terms of use}
for images and the \href{https://cocodataset.org/#termsofuse}{CC BY 4.0} for annotations.
MMDetection (\href{https://github.com/open-mmlab/mmdetection}{\texttt{https://github.com/open-mmlab/mmdetection}}) and MMSegmentation (\href{https://github.com/open-mmlab/mmsegmentation}{\texttt{https://github.com/open-mmlab/mmsegmentation}}) codebases are released under the \href{apache.org/licenses/LICENSE-2.0}{Apache-2.0 license}. Openseg.pytorch (\href{https://github.com/openseg-group/openseg.pytorch}{\texttt{https://github.com/openseg-group/openseg. pytorch}}) codebase is  released under the \href{https://opensource.org/license/mit/}{MIT license}. 

\textbf{Limitations.} The presented  affinity propagation method is performed under the guidance of the similarities of image intensity and color.
Our proposed method may have difficulties in accurately capturing the pairwise affinities under the challenging scenarios like motion blur, occlusions, and cluttered scenes, \textit{etc}. Actually, this is a common problem for many segmentation methods.
In the future work, we will explore how to integrate our method into the large-scale foundation models, such as SAM~\cite{kirillov2023sam}, to take advantage of their strong features for more promising segmentation results.

\textbf{Broader Impact.}  This work presents an effective component for weakly-supervised segmentation with label-efficient annotations. We have demonstrated its effectiveness over three typical label-efficient segmentation tasks.
On the positive side, our approach has the potential to benefit a wide variety of real-world applications, such as autonomous vehicles, medical imaging, remote sensing and image editing, which can significantly reduce the labeling costs.
On the other side, erroneous predictions in real-world applications (\textit{i.e.}, medical imaging analysis and tasks involving autonomous vehicles) raise the safety issues of human beings. In order to avoid the potentially negative effects, we suggest to adopt a highly stringent security protocol in case that our approach 
fails to function properly in real-world applications.

\vspace{14pt}
\begin{figure}[H]
\begin{center}
\includegraphics[width=0.97\linewidth]{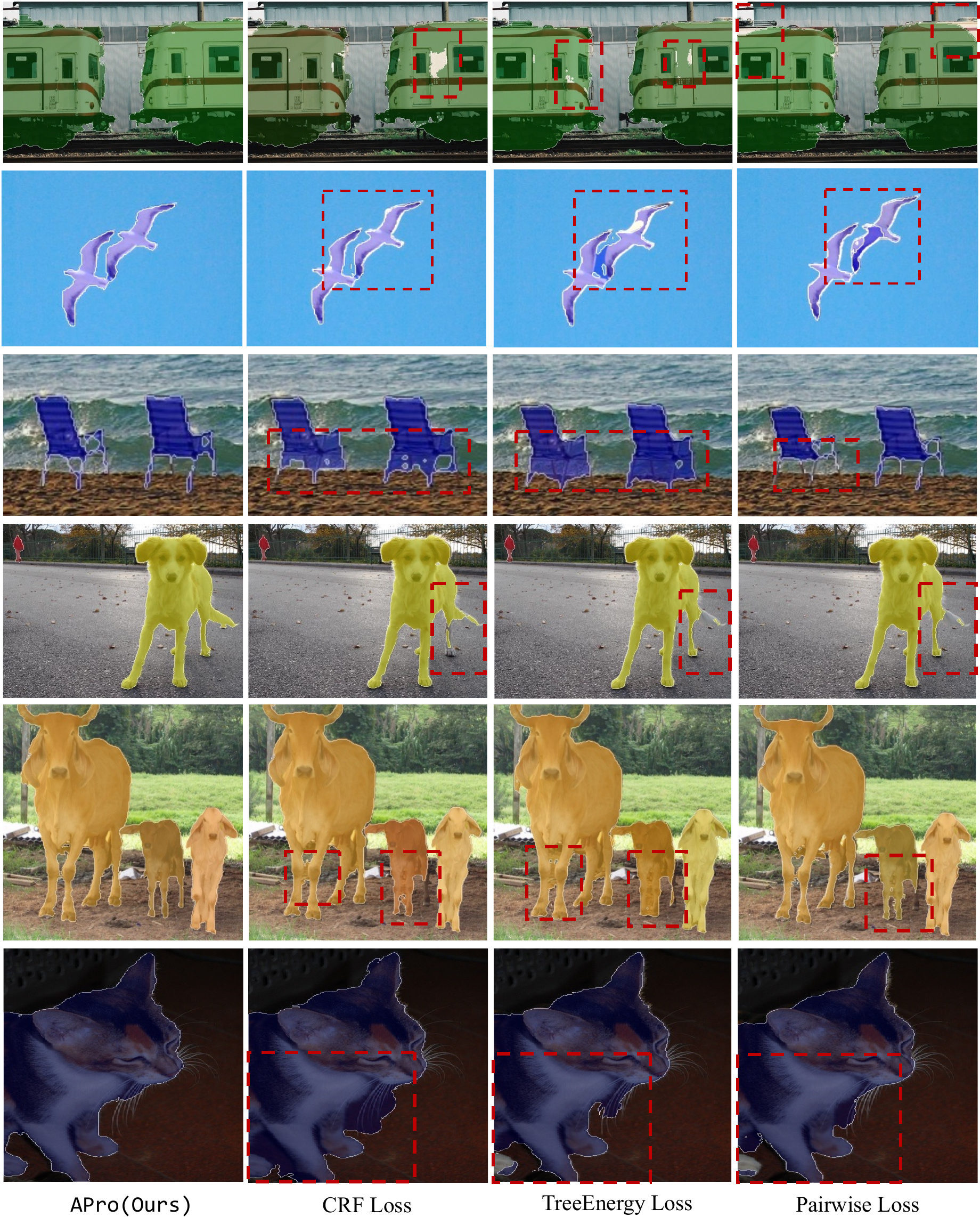} \end{center}
    \caption{Qualitative comparisons on Pascal VOC~\cite{pascalvoc2010}. We compare our \texttt{APro} approach with CRF loss~\cite{CVPR2023MAL}, TreeEnergy loss~\cite{liang2022tree} and Pairwise loss~\cite{cvpr2021boxinst} under the SOLOv2~\cite{nips2020solov2} framework. Our method obtains more fine-grained predictions with detail preserved.}
 \label{fig:box_voc}
 \end{figure}

\begin{figure}[t]
\begin{center}
\includegraphics[width=0.96\linewidth]{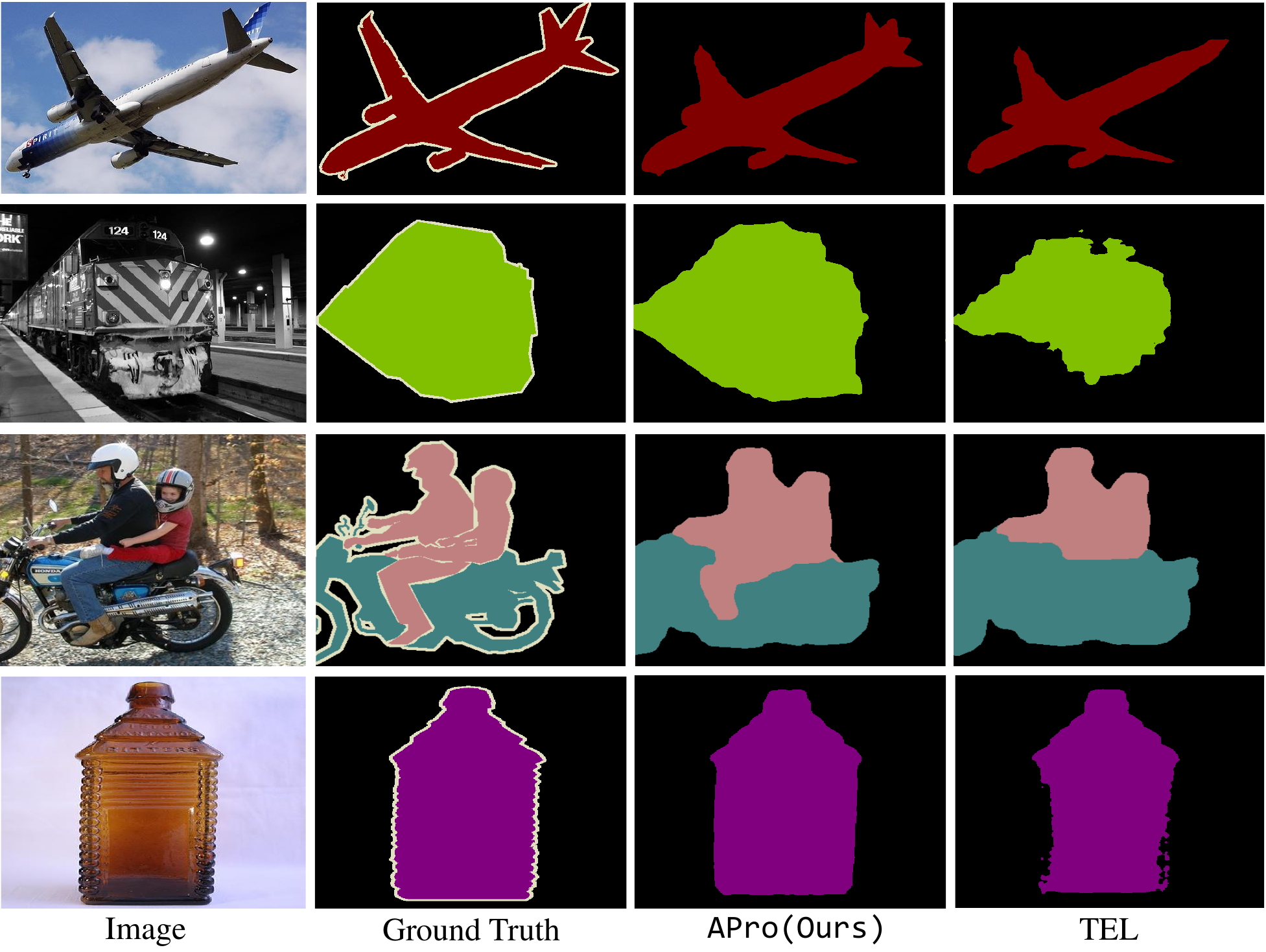} \end{center}
   \vspace{-8pt}
    \caption{Qualitative comparisons on point-supervised semantic segmentation. Compared with the state-of-the-art TEL~\cite{liang2022tree}, our method  segments objects with more accurate boundaries.}
 \label{fig:point_sem}
 \end{figure}

\vspace{-10pt}
\begin{figure}[t]
\begin{center}
\includegraphics[width=0.96\linewidth]{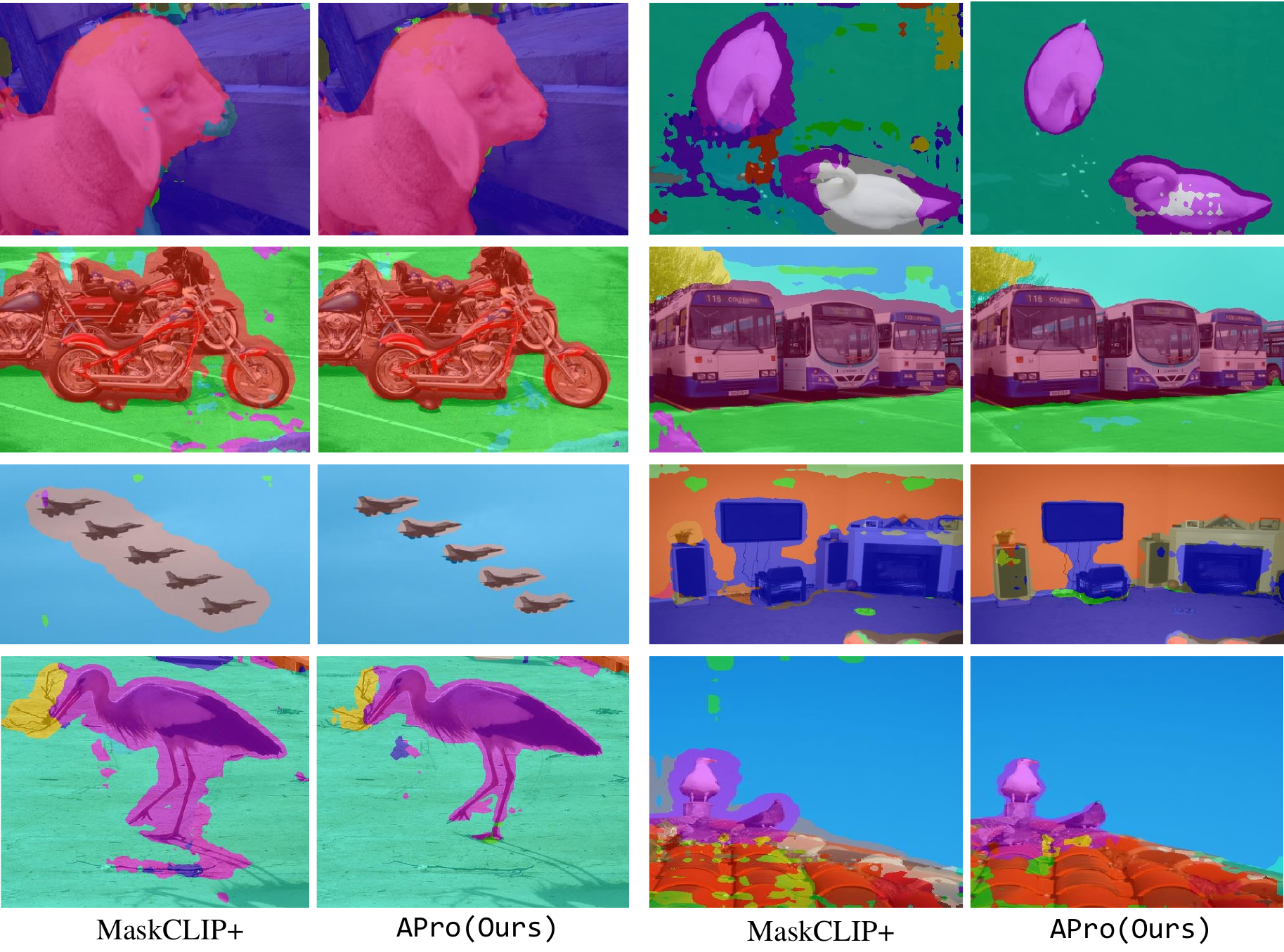} \end{center}
   \vspace{-8pt}
    \caption{Visual comparison results on Pascal Context with ViT-B/16 image encoder. Compared with the prior art MaskCLIP+~\cite{zhou2022extract}, our method obtains more accurate predictions with fitting boundaries.}
 \label{fig:clip-guide}
 \end{figure}

\begin{figure}[t]
\begin{center}
\includegraphics[width=0.99\linewidth]{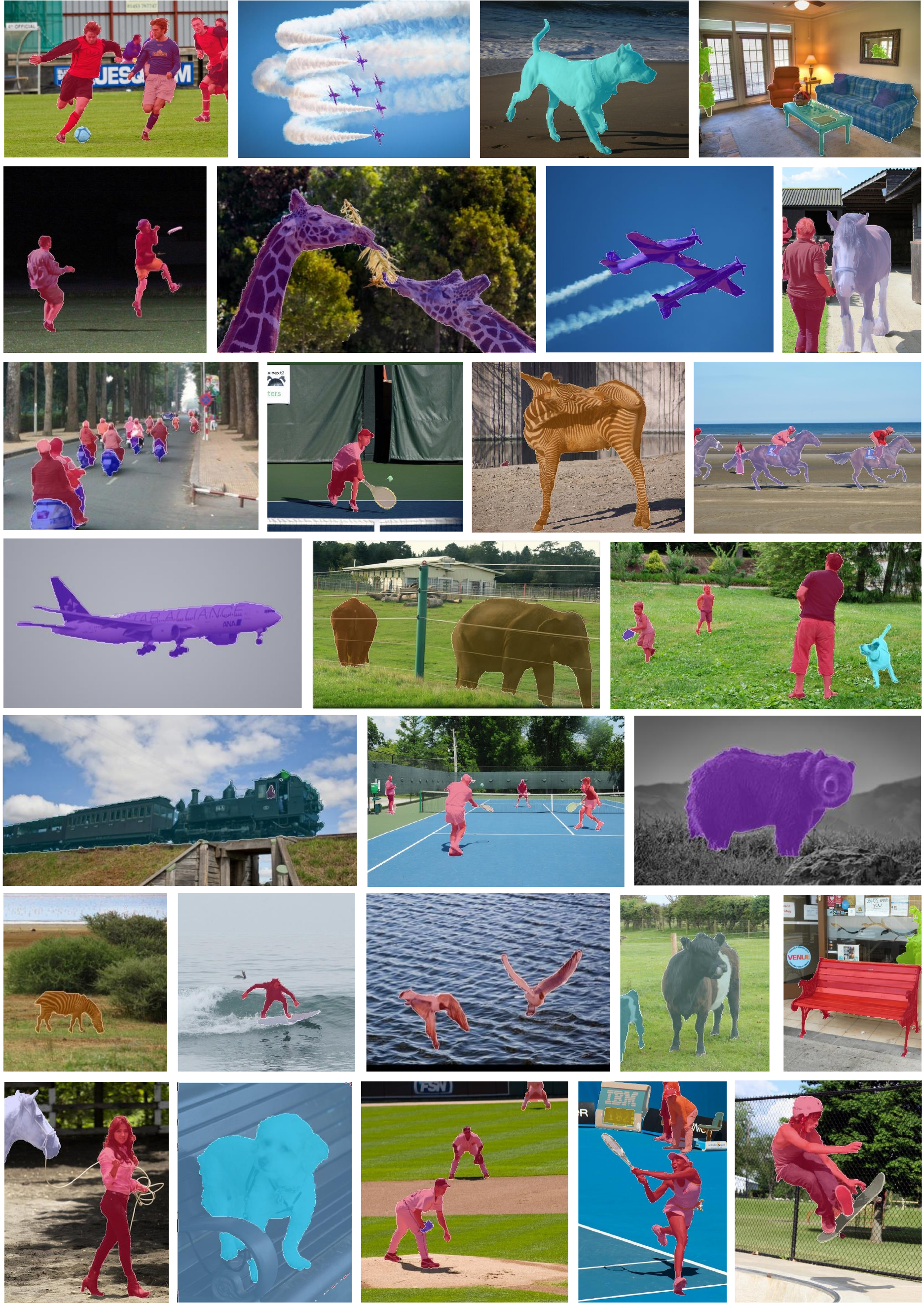} \end{center}
    \caption{Qualitative results of our \texttt{APro} on COCO with ResNet-101 under the SOLOv2 framework upon box-supervised instance segmentation.}
 \label{fig:vis_coco}
\end{figure}



\end{document}